%% file: main.tex
\documentclass{article}
\usepackage{iclr2027_conference,times}
\usepackage[T1]{fontenc}
\usepackage{amsmath,amssymb,amsthm,mathtools,booktabs,array}
\usepackage{hyperref,url}
\hypersetup{hidelinks,pdftitle={Prediction Is Not a Sufficient Memory State},pdfauthor={Robert Jenkinson Alvarez}}
\newtheorem{theorem}{Theorem}
\newtheorem{proposition}[theorem]{Proposition}
\newtheorem{lemma}[theorem]{Lemma}

\newcommand{\R}{\mathbb R}\newcommand{\E}{\mathbb E}
\newcommand{\sym}{\operatorname{sym}}\newcommand{\Sym}{\operatorname{Sym}}\newcommand{\spanop}{\operatorname{span}}
\newcommand{\rank}{\operatorname{rank}}\newcommand{\op}{\mathrm{op}}\newcommand{\F}{\mathrm F}
\newcommand{\argmin}{\operatorname*{arg\,min}}\newcommand{\G}{\mathcal G}\newcommand{\cK}{\mathcal K}
\newcommand{\norm}[1]{\left\lVert#1\right\rVert}\newcommand{\inner}[2]{\langle #1,#2\rangle}
\title{Prediction Is Not a Sufficient Memory State}
\author{
\texorpdfstring{
Robert Jenkinson Alvarez\thanks{
\href{https://r02aj.github.io}{\texttt{r02aj.github.io}}
}
}{
Robert Jenkinson Alvarez
}
}
\iclrfinalcopy
\begin{document}\maketitle
\fancyhead{}
\pagestyle{plain}
\input{sections/01_intro}
\newpage\input{sections/02_ridge}
\newpage\input{sections/03_native}
\newpage\input{sections/04_closure}
\newpage\input{sections/05_logistic}
\newpage\input{sections/06_bits}
\newpage\input{sections/07_attention}
\newpage\input{sections/08_approx}
\newpage\input{sections/09_discussion}
\label{main:end}
\clearpage\input{statements}
\clearpage
\bibliographystyle{iclr2027_conference}
\bibliography{references}
\clearpage\appendix
\numberwithin{equation}{section}
\input{appendices/01_native}
\input{appendices/02_closure}
\input{appendices/03_ridge}
\input{appendices/04_logistic}
\input{appendices/05_bits}
\input{appendices/06_attention}
\input{appendices/07_approx}
\input{appendices/08_coherence}
\input{appendices/10_complete_attention}
\input{appendices/11_observable_coordinates}
\input{appendices/12_response_geometry}
\input{appendices/13_reusable_execution}
\input{appendices/14_information_fusion}
\input{appendices/09_scope}
\end{document}

%% file: sections/01_intro.tex
\begin{abstract}
A memory that reproduces every current prediction may still have forgotten how to learn from the next observation. We characterize this distinction for exact cumulative learning and attention. For smooth convex residual losses with unrestricted real features and labels, a continuous recurrent state of fixed dimension exists exactly when the loss is polynomial. With ordinary, unweighted observations, degree $D$ requires precisely $r=\binom{d+D}{D}-1$ coordinates; even among histories with one fixed current predictor, exactly $r-d$ coordinates can remain necessary for continuation. Nonpolynomial losses have unbounded requirements on such equal-prediction histories. Bounded scalar logistic examples retain this obstruction after fixing any finite derivative jet. For ridge regression, the hidden evidence is the design matrix: the exact state costs $dp+d(d+1)/2$, versus $dp$ for the predictor. On a well-conditioned zero-prediction family, optimal $B$-bit compression has average one-update error of order $d^{-1}2^{-2B/s_d}$, $s_d=d(d+1)/2$, even with arbitrary randomized codes. Kernel attention gives a second exact closure classification. The results identify what future updates reveal, rather than proposing a new regression algorithm or asserting trained-model gains.
\end{abstract}
\section{What must a memory remember to remain able to learn?}
A test-time learner can store a predictor in its recurrent state and use incoming examples to modify it \citep{sun2024,wang2026}, but representing the current predictor and representing a state from which that predictor can be updated are different tasks. The difference persists even when every present query is answered exactly, not merely when a finite evaluation set is reconstructed.

Consider regression on zero labels. Its current prediction is zero for any design matrix. Nevertheless, two such histories can respond differently to the same next nonzero label. In normalized attention, zero values similarly conceal all the past keys until a new nonzero value makes their normalization mass observable. Current-output compression has removed information that the update still needs.

We ask the question: \emph{which histories may an exact learner identify without changing any response to an allowed continuation?} Classical state distinguishability and predictive-state representations already organize memory by future behavior \citep{littman2001}. Our contribution is to solve this identification problem for native learning updates, including the amount of information invisible on an entire equal-prediction family. Prediction in the title means the present readout, not a predictive state containing predictions of all future experiments.

The central results are a native, unweighted loss classification and a quantitative finite-bit law. Polynomial losses close in finite additive state; every nonpolynomial smooth convex residual loss fails exact fixed-size closure, even after fixing the current optimum. The ridge specialization shows that this is not only an infinite-precision distinction: preserving future adaptation to random examples retains a quadratic-dimensional error exponent. Attention then identifies the parallel role of unnormalized kernel evidence. All results concern the specified computations; changing their objective, feature domain, precision, or accuracy criterion changes the problem.

%% file: sections/02_ridge.tex
\section{Ridge regression: the cost of being ready to learn}\label{sec:ridge}
Let $h=((x_i,y_i))_{i=1}^N$, $x_i\in\R^d$, $y_i\in\R^p$, and $\lambda>0$. Define
\begin{equation}\label{eq:ridge}
 C_h=\lambda I+\sum_i x_ix_i^\top,\qquad M_h=\sum_i y_ix_i^\top,\qquad W_h=M_hC_h^{-1}.
\end{equation}
The readout at $q\in\R^d$ is $W_hq$, the unique cumulative ridge prediction. A \emph{prefix summary} is a continuous map of a fixed-length observed history into $\R^m$. Every history-dependent auxiliary array is included; static parameters and a known clock are allowed. The decoder may be arbitrary. A summary supports one continuation when it returns the new readout from only its state, one new observation, and the query. An exact recurrent learner necessarily supports this operation at every prefix.

\begin{theorem}[Exact ridge-state dimension]\label{thm:ridge}
The minimum continuous read-only dimension is $dp$. Supporting all present readouts and one arbitrary new ridge observation requires exactly
\begin{equation}\label{eq:ridge_dimension}
 m=dp+s_d,\qquad s_d=\frac{d(d+1)}2.
\end{equation}
On zero-label prefixes, where all present readouts are identically zero, the exact continuation-summary dimension is $s_d$. The full bound is attained by $(C_h,M_h)$ and suffices for every finite continuation.
\end{theorem}
The attaining algorithm is recursive least squares which is not a new architecture. The lower bound applies to arbitrary continuous nonlinear states computing the same objective, not only to matrix-valued implementations.

\paragraph{Native responses recover the missing evidence.}
Let $P=C^{-1}$. The ordinary update $(k,y)$ gives
\begin{equation}\label{eq:gain}
 W^+=W+(y-Wk)\Phi_P(k)^\top,\quad
 \Phi_P(k)=\frac{Pk}{1+k^\top Pk},\quad
 Pk=\frac{\Phi_P(k)}{1-k^\top\Phi_P(k)}.
\end{equation}
Equal states must first give equal $W$. For the common next label $y=Wk+e_1$, equal future readouts reveal equal $\Phi_P(k)$ and hence equal $Pk$. Basis keys identify $P$, then $M=WC$. Every $C\succ\lambda I$ and every $M$ is continuously reachable with exactly $d$ observations: take column matrices $X=(C-\lambda I)^{1/2}$ and $Y=MX^{-\top}$. An open $(s_d+dp)$-dimensional history family must therefore inject into the state. Invariance of domain supplies the dimension bound. Appendix~\ref{app:ridge} proves the construction and all three minimum-dimension statements.

\paragraph{A concrete equal-prediction collision example.}
For $d=2,p=1,\lambda=1/4$, take zero-label histories with
\begin{equation}\label{eq:collision}
 C_\pm=\begin{pmatrix}1&\pm1/2\\\pm1/2&1\end{pmatrix}.
\end{equation}
Their counts, diagonal evidence, and predictions agree. Appending $(e_1,1)$ gives predictor rows $(4/7,\mp2/7)$. A merged state has average and worst-case squared error at least $4/49$ at query $e_2$. The raw delta recurrence $S^+=S+\beta(y-Sk)k^\top$, initialized at zero, merges these histories regardless of key-dependent scalar $\beta$. This is \textbf{not} a claim about all learned DeltaNet variants: additional value channels or hidden states can carry the evidence, and must be counted.

Unit-normalizing keys does not remove the quadratic requirement. With known length $N\ge d$ and $d\ge2$, the exact full dimension is $dp+s_d-1$: only the known trace coordinate disappears (Appendix~\ref{app:unit}).

%% file: sections/03_native.tex
\section{A native state theorem without weighted observations}\label{sec:native}
The ridge formula raises the following question: is the required memory always a predictor plus a curvature matrix? Consider ordinary, \emph{unweighted} residual observations $z=(x,y)\in\R^d\times\R$ and
\begin{equation}\label{eq:objective}
 F_h(w)=R(w)+\sum_{z\in h}\rho(x^\top w-y),\qquad w_h=\argmin_w F_h(w).
\end{equation}
Throughout Sections~\ref{sec:native}--\ref{sec:poly}, $R$ is fixed, $C^\infty$, and $\lambda$-strongly convex on $\R^d$; $\rho$ is $C^\infty$, convex, and nonconstant. Features and labels range over all Reals. Every objective is coercive and has a unique optimum. The present predictor is $q\mapsto q^\top w_h$.

Define the space of gradient \emph{functions}, not gradients at one parameter,
\begin{equation}\label{eq:span}
 g_z(w)=x\rho'(x^\top w-y),\qquad
 \G_\rho=\spanop\{g_z:z\in\R^{d+1}\},\qquad r=\dim\G_\rho.
\end{equation}
The distinction between this function space and $\R^d$ determines the state size.

\begin{theorem}[Native identification and unweighted state]\label{thm:native}
Two histories have identical optima after every single additional observation if and only if their accumulated gradient functions $\sum_{z\in h}g_z$ agree. If $r<\infty$, the exact minimum continuous recurrent state dimension is $r$, attained by additive gradient coefficients. If $r=\infty$, no fixed finite dimension suffices. All prefixes and continuations use unit observation weights.
\end{theorem}
\paragraph{One native observation is enough to identify a function.}
Choose $t_0$ with $c=\rho'(t_0)\ne0$. Given a history and any proposed next optimum $w$, append
\begin{equation}\label{eq:steer}
 x=-\nabla F_h(w)/c,\qquad y=x^\top w-t_0.
\end{equation}
Its loss gradient at $w$ is $-\nabla F_h(w)$, so $w$ becomes the unique optimum. If the same appended example has the same optimum for another history, their old gradients agree at $w$. Varying $w$ identifies the whole profile. This is a distinguishability proof where the decoder is not supplied with $\nabla F_h$, a target optimum, or an extra oracle. The compared histories receive the same ordinary observation.

\paragraph{Why unweighted histories still fill an open state family.}
Varying positive sample weights would give a linear coefficient chart, but ordinary observations have no such input. Instead, vary the observations themselves. Since $g_{(0,y)}=0$, the observation tangents $D_zg_z[v]$ span all of $\G_\rho$ when it is finite dimensional, and have unbounded span otherwise. A finite tangent span contains every $g_z$ by integrating along a path from $(0,y)$ to $z$; conversely derivatives remain in a finite-dimensional function space.

Choose independent tangents at $r$ observations. Finitely many point evaluations separate them. Varying each observation along its chosen tangent gives an invertible derivative for those evaluated accumulated gradients. The inverse-function theorem yields an open, distinguishable $r$-parameter family at one fixed length. Continuity of the encoder and invariance of domain imply $m\ge r$. A finite coefficient basis attains the bound. Appendix~\ref{app:native} proves the evaluation, tangent, and continuity arguments, including the infinite-dimensional case.

\paragraph{One-step equivalence is already closed under future updates.}
Let $h\sim_0h'$ mean equality of all present predictions, and $h\sim_1h'$ equality after every single native observation. Theorem~\ref{thm:native} identifies $\sim_1$ with equality of gradient profiles. Adding the same $g_z$ preserves that equality, so $h\sim_1h'$ implies $hz\sim_1h'z$. Induction gives equality after every finite continuation. The theorem characterizes an update-closed state, not unrelated one-step answers. The present relation $\sim_0$ need not have this property.

%% file: sections/04_closure.tex
\section{Which losses close, and how much is invisible now?}\label{sec:poly}
\begin{theorem}[Polynomial closure and equal-prediction memory]\label{thm:poly}
Under \eqref{eq:objective}, an exact continuous recurrent state of size independent of history length exists if and only if $\rho$ is a polynomial. For degree $D\ge1$, its minimum dimension is
\begin{equation}\label{eq:poly_count}
 r_D(d)=\binom{d+D}{D}-1.
\end{equation}
The present predictor alone has minimum dimension $d$. For every fixed $w_0\in\R^d$, the minimum dimension of a continuous prefix summary supporting one continuation on the family $\{h:w_h=w_0\}$ is exactly $r_D(d)-d$. For every nonpolynomial $\rho$, the required dimension is unbounded even on $\{h:w_h=w_0\}$. A constant loss requires no history state.
\end{theorem}
The fixed-optimum statement is about compression \emph{at those prefixes}. After a new example moves the predictor, a recurrent implementation may need the full $r_D(d)$ coordinates. It is important that this not read as an assertion that an $(r_D(d)-d)$-coordinate state works on every history.

\paragraph{Finite affine-orbit dimension forces polynomial structure.}
If $\G_\rho$ is finite dimensional, restricting $x=ae_1$ and $w=te_1$ makes
\[
 V=\spanop\{t\mapsto\rho'(at+b):a>0,\ b\in\R\}
\]
finite dimensional. Translations imply that differentiation $\mathsf D$ is an endomorphism of $V$. Dilations $T_af(t)=f(at)$ obey $\mathsf DT_a=aT_a\mathsf D$. Thus $\mathsf D$ is similar to every positive multiple of itself. Its finite complex spectrum can contain only zero, so Cayley--Hamilton makes it nilpotent. Every function in $V$, and therefore $\rho$, is polynomial.

Conversely, translations of a degree-$D$ polynomial span every univariate degree at most $D$. Varying the feature vector spans all multivariate polynomials of total degree at most $D$. Taking gradients removes exactly the constant direction, proving \eqref{eq:poly_count}. These function-space ingredients have classical antecedents \citep{anselone1964,mattner2000}; the native state and equal-prediction assertions are the applications proved here.

\paragraph{Fixing the predictor removes exactly $d$ dimensions, not all evidence.}
Start with the unweighted open profile chart from Theorem~\ref{thm:native}. For each history, append the balancing example \eqref{eq:steer} with $w=w_0$. Every balanced history now has optimum $w_0$. The balancing profile depends on the original chart only through $\nabla F_h(w_0)\in\R^d$. Its derivative has rank at most $d$, so adding it can reduce the original profile rank by at most $d$. This constructs a distinguishable $(r-d)$-dimensional equal-prediction family. Arbitrarily large charts prove the nonpolynomial statement.

For polynomial losses, $\nabla F_h(w_0)=0$ imposes exactly $d$ independent affine conditions on gradient coefficients: constant vector gradients already span $\R^d$. Storing the remaining coefficients attains $r-d$. Appendix~\ref{app:poly} proves every rank and reachability step. For squared scalar regression, $r_2=d+s_d$ and the invisible part is $s_d$. For fourth-power loss it is $\binom{d+4}{4}-d-1$. A fixed Hessian is not a general finite-state completion of a predictor.

\paragraph{Sharp boundary of the classification.}
With scalar feature $x=1$ fixed and $\rho(t)=e^t$, the entire history gradient is $e^w\sum_i e^{-y_i}$. One real coordinate supports exact continued learning despite the nonpolynomial loss. Allowing arbitrary feature scales replaces this span by $\spanop\{e^{aw+b}\}$, which is infinite dimensional. Thus the observation domain, not the loss name alone, determines finite closure. The theorem's unrestricted real feature and label domain is substantive; bounded or discrete domains require separate analysis.

%% file: sections/05_logistic.tex
\section{Bounded logistic data: even a finite local jet is insufficient}\label{sec:logistic}
The native classification uses unrestricted features to steer a proposed optimum. The following obstruction needs neither unbounded features nor continuously weighted examples. Consider scalar logistic regression with regularization $\lambda w^2/2$ and observations $(x,y)$, $y\in\{-1,+1\}$, of loss $\log(1+e^{-yxw})$.

For $1<a_1<\cdots<a_N<2$, take the $2N$ ordinary observations $(a_i,+1)$ and $(a_i,-1)$. Every prefix has optimum zero and therefore the same prediction at every query. Its objective derivative is
\begin{equation}\label{eq:logistic_profile}
 F_a'(w)=\lambda w+\sum_{i=1}^N a_i\tanh(a_iw/2).
\end{equation}
\begin{theorem}[Bounded native logistic obstruction]\label{thm:logistic}
For every $b_0>0$, any continuous prefix summary that returns the exact optimum after every one-example continuation $(b,+1)$, $0<b<b_0$, requires at least $N$ coordinates on this family. Thus no fixed continuous state computes exact cumulative scalar logistic regression for arbitrary history lengths.
\end{theorem}
\paragraph{A small response curve identifies all past feature scales.}
Let $w_a(b)$ denote the updated optimum. It solves
\begin{equation}\label{eq:logistic_update}
 F_a'(w_a(b))=\frac{b}{1+e^{bw_a(b)}},\qquad
 w_a'(0)=\frac{1}{2(\lambda+\frac12\sum_i a_i^2)}>0.
\end{equation}
The analytic implicit-function theorem gives a nondegenerate response interval near zero. Two histories with identical responses therefore have identical $F_a'$ on an interval, and hence everywhere by analyticity. The positive tail
\begin{equation}\label{eq:logistic_tail}
 \sum_i a_i-\sum_i a_i\tanh(a_iw/2)=\sum_i\frac{2a_i}{1+e^{a_iw}}
\end{equation}
identifies the smallest $a_i$ by its exponential decay as $w\to+\infty$. Subtracting its full contribution identifies the next, and so on. This injects the open $N$-parameter family into any exact state. The tail is used to prove uniqueness, not as an allowed large-feature continuation.

\begin{theorem}[Equal finite jets, different future learning]\label{thm:jet}
Fix an integer $K\ge0$ and $N>K$. There are local $(N-K)$-dimensional families of the same bounded, unweighted $2N$-observation prefixes whose objective derivatives of every order $0,\ldots,2K+1$ at zero agree, yet whose one-example continuation responses distinguish every pair. Any exact continuous summary on such a family requires at least $N-K$ coordinates.
\end{theorem}
For $K\ge1$, the displayed derivatives are determined by the first $K$ power sums $\sum_i a_i^{2j}$; odd derivatives vanish and $F_a(0)=2N\log2$. Their Jacobian has rank $K$ by a Vandermonde argument. A regular level set therefore has dimension $N-K$, while Theorem~\ref{thm:logistic} still distinguishes its histories. Appendix~\ref{app:logistic} proves this without weighted samples.

A Hessian describes first-order adaptation to a small update; any fixed higher jet still misses some native nonlinear continuations. The result doesn't provide an $N$-independent numerical error floor: an injective analytic response curve can be poorly conditioned. Approximate logistic summaries and coresets remain possible under their own guarantees \citep{munteanu2018,huggins2017}. The finite-precision result next addresses that gap in a different, uniformly conditioned model rather than inferring robustness from exact identifiability.

%% file: sections/06_bits.tex
\section{A sharp average-error exponent for finite-bit memory}\label{sec:bits}
Exact topological bounds alone do not price approximate prediction. We now allow any measurable encoder into at most $2^B$ states, including randomized encoders, with no continuity requirement. The decoder is unrestricted. The next example and read query are ordinary random unit vectors, independent of the prefix.

Fix $d\ge2$, $r>0$, $a>18r$, and $b=a+18r<1/\lambda$. Let
\begin{equation}\label{eq:body}
 \cK_d=\{H\in\Sym_d:\norm H_\F\le r\sqrt d,\ \norm H_\op\le18r\},
 \qquad P=aI+H,
\end{equation}
with $H$ uniform in $\cK_d$. Realize $C=P^{-1}$ by $d$ zero-label ridge observations as in Section~\ref{sec:ridge}. All prefixes have the same length, zero current predictor, and uniformly bounded condition number. Append $(k,1)$ and query $q$, where $k,q$ are independent and uniform on the unit sphere. The target is
\[
 f_P(k,q)=q^\top\Phi_P(k),\qquad
 \mathcal R_B=\inf_{S,\widehat f}\E\bigl[\widehat f(S,k,q)-f_P(k,q)\bigr]^2.
\]
\begin{theorem}[Matching average finite-bit law]\label{thm:bits}
For every integer $B\ge0$, with $s_d=d(d+1)/2$,
\begin{equation}\label{eq:bits}
 \frac{r^2}{4d(1+b)^4}\,2^{-2B/s_d}
 \ \le\ \mathcal R_B\ \le\
 \frac{16(1+b)^2r^2}{d}\,2^{-2B/s_d}.
\end{equation}
The lower bound allows private randomization and independent shared randomness. The upper bound is an offline prefix quantizer with a coherent ridge decoder; no efficient streaming quantizer is asserted.
\end{theorem}
A sufficiently small fixed normalized error requires a quadratic number of bits. Further reducing the normalized error costs $\Theta(s_d\log(1/\varepsilon))$ bits. The factor $1/d$ in raw risk is retained: this is an explicit random-probe problem, not an unnormalized constant-error claim for language modeling.

\paragraph{From arbitrary predictions back to matrix distortion.}
For any $P,Q$ in the family, gain inversion and a forward estimate give
\begin{equation}\label{eq:gain_distance}
 \frac{\norm{P-Q}_\F^2}{d^2(1+b)^4}
 \le \E_{k,q}[f_P-f_Q]^2
 \le \frac{(1+b)^2\norm{P-Q}_\F^2}{d^2}.
\end{equation}
This initially compares coherent states, but it does not restrict the learned decoder. Its Bayes-optimal output is the conditional mean of $f_P$. Introducing two independent posterior draws $P,P'$ given the code turns minimum prediction risk into half their mean squared response difference. Equation~\eqref{eq:gain_distance} then lower-bounds it by the posterior matrix variance divided by $d^2(1+b)^4$.

The body in \eqref{eq:body} occupies at least half the volume of its Frobenius ball. A finite-code moment inequality gives matrix distortion at least $(r^2d/4)2^{-2B/s_d}$. A covering quantizer gives the matching exponent. Appendix~\ref{app:bits} proves the volume bound, posterior reduction, randomized-code inequality, and upper construction. This is separate from statistical streaming-regression lower bounds \citep{brown2022}.

\paragraph{What the exponent prices.}
Writing $\varepsilon_B=d\mathcal R_B/r^2$, the theorem bounds $\varepsilon_B$ above and below by constants depending only on $b$ times $2^{-2B/s_d}$. The code is fixed before the next key and query are known. Supplying both probes \emph{before} compression and storing their one desired answer would pose a different problem. Computation after the probes arrive does not restore evidence already merged by the prefix code.

%% file: sections/07_attention.tex
\section{Attention: the ratio is not the continuation state}\label{sec:attention}
For a strictly positive kernel $\kappa$, a nonempty history has
\begin{equation}\label{eq:attention}
 Z_h(q)=\sum_i t_i\kappa(q,k_i),\quad
 B_h(q)=\sum_i t_i\kappa(q,k_i)v_i,\quad f_h(q)=B_h(q)/Z_h(q),
\end{equation}
where $v_i\in\R^p$. We specify below when weights $t_i>0$ are inputs and when every $t_i=1$. Appending a token adds to the unnormalized pair.

\begin{proposition}[One-token identification]\label{prop:attention}
Two histories have identical responses after every single new token if and only if their functions $(B_h,Z_h)$ agree. One fixed appended key, used with the two values $0$ and $e_1$, identifies this pair.
\end{proposition}
At query $q$, write $\alpha=\kappa(q,k_*)>0$ and let $u_0,u_1$ be the two continued outputs. Their first-coordinate difference is $\delta=\alpha/(Z+\alpha)$, giving
\begin{equation}\label{eq:att_ident}
 Z=\alpha(\delta^{-1}-1),\qquad B=(Z+\alpha)u_0.
\end{equation}
Standard linear attention already stores numerator and denominator statistics \citep{katharopoulos2020}. The point is their necessity against nonlinear compression, not their invention.

\begin{theorem}[Kernel closure and native softmax size]\label{thm:attention}
Let $r_\kappa=\dim\spanop\{\kappa(\cdot,k):k\in K\}$, with $k\mapsto\kappa(q,k)$ continuous for every $q$. For positively weighted histories with values in $\R^p$, the exact minimum continuous continuation-state dimension is $r_\kappa(p+1)$ when finite; otherwise no fixed finite dimension suffices.

For \emph{unweighted} exponential attention $\kappa(q,k)=e^{q^\top k}$, with queries in a nonempty open subset of $\R^d$, a length-$N$ prefix with distinct real keys and arbitrary values requires exactly $N(d+p)$ coordinates for one-token continuation. On the zero-valued subfamily the exact count is $Nd$.
\end{theorem}
A kernel basis gives the additive upper state. For the weighted lower bound, choose independent sections and vary their positive masses and weighted values. The resulting $r_\kappa(p+1)$-parameter family is distinguishable. The weights are substantive; for unweighted fixed-length histories, a known constant-kernel component can remove a count coordinate.

For exponential attention, place keys in disjoint open balls to fix their indices continuously. Distinct real exponentials are independent on an open query domain. The pair $(B,Z)$ therefore identifies every key and value, giving an open $N(d+p)$-parameter family. Raw storage attains the prefix bound. On zero values, present predictions reveal nothing, but one nonzero appended value reveals the key measure. Appendix~\ref{app:attention} proves every statement.

This is an exact continuous bound. Approximate streaming attention has precision-specific upper and lower bounds \citep{kochetkova2025,chen2026}. Finite vocabularies, restricted queries, and task-average error can change the relevant scale.

\begin{table}[h!]
\caption{Exact continuous state and evidence hidden at fixed predictions.}
\label{tab:state}
\centering
\begin{tabular}{@{}lcc@{}}\toprule
Computation & Exact state & Hidden evidence\\\midrule
Ridge, $p$ outputs & $dp+s_d$ & $s_d$ at zero labels\\
Degree-$D$ residual loss & $r_D(d)$ & $r_D(d)-d$\\
Logistic, $2N$ paired samples & unbounded in $N$ & at least $N$\\
Softmax, $N$ distinct keys & $N(d+p)$ & $Nd$ at zero values\\\bottomrule
\end{tabular}
\end{table}
The last column concerns prefix summaries on the indicated family, not a smaller global recurrence. The softmax count is at a fixed length. All entries measure continuous exact state, unlike Theorem~\ref{thm:bits}.

%% file: sections/08_approx.tex
\section{Preserve the evidence that the allowed future will use}\label{sec:approx}
The exact results do not recommend storing an infinite object in practice. They specify what an approximation must control. Present loss can vanish while future loss varies across a large family; a certificate on accumulated evidence is a different guarantee.

\begin{proposition}[A gradient-profile certificate survives common updates]\label{prop:grad_approx}
Let $F$ be $\lambda$-strongly convex and $U$ any common differentiable convex continuation objective. Suppose the true minimizer $w^+$ of $F+U$ and a returned point $\widehat w$ lie in a domain $\Omega$. If
\[
 \sup_{w\in\Omega}\norm{\widehat G(w)-\nabla F(w)}\le\delta,
 \qquad\norm{\widehat G(\widehat w)+\nabla U(\widehat w)}\le\eta,
\]
then $\norm{\widehat w-w^+}\le(\delta+\eta)/\lambda$.
\end{proposition}
Strong monotonicity of $\nabla(F+U)$ proves the statement. It does not require $\widehat G$ to be a gradient or an approximate objective to be convex. It does require a uniform error domain containing both points and a small residual at the returned point. A small Taylor remainder at the old optimum alone does not establish these conditions.

For bounded features and residuals, a polynomial approximation $p$ to $\rho'$ with uniform error $\varepsilon$ yields a finite moment summary with
\begin{equation}\label{eq:moments}
 \sup_{w\in\Omega}\norm{\sum_i x_i[p(x_i^\top w-y_i)-\rho'(x_i^\top w-y_i)]}
 \le\varepsilon\sum_i\norm{x_i}.
\end{equation}
This connects finite-state approximation to polynomial approximate statistics \citep{huggins2017}. The growth with total feature mass is not discarded; normalized or forgetting objectives need their own scaling. Exact nonpolynomial impossibility does not rule out useful bounded-domain approximations.

\begin{proposition}[Attention evidence errors under continuation]\label{prop:att_approx}
At a query, suppose $|\widehat Z-Z|\le\delta Z$, $\norm{\widehat B-B}\le\eta Z$, $0\le\delta<1$, and all old and new values have norm at most $V$. After any common continuation of kernel mass $A\ge0$,
\begin{equation}\label{eq:att_approx}
 \norm{\widehat f^+-f^+}\le
 \frac{\theta(\eta+\delta V)}{1-\theta\delta},\qquad
 \theta=\frac{Z}{Z+A}\le1.
\end{equation}
\end{proposition}
Subtracting the two ratios proves the claim. Error in the old ratio alone gives no comparable guarantee: $B=0$ hides arbitrary positive $Z$. Both propositions are conditional approximation results, not optimizer convergence or generalization theorems. Full proofs and a controlled local-curvature bound appear in Appendix~\ref{app:approx}.

\paragraph{An evaluation that separates prediction from retained learning.}
Construct histories with identical current readouts, vary their hidden evidence, then append the same observation and compare the resulting outputs. For ridge, full RLS statistics give an exact reference; diagonal or sketched evidence, pure error-driven updates, and replay are informative controls when their entire history-dependent state is counted. Random probes test Theorem~\ref{thm:bits}; bounded logistic updates test a different, qualitative obstruction. A trained-model study must show that its actual update target matches these computations, not infer an architectural failure from a label such as ``memory'' or ``test-time training.''

Evidence preservation also changes how compression is supervised. Reconstructing the present predictor constrains none of the hidden directions on an equal-prediction family. A continuation loss can distinguish those directions. This is an identification principle, not a theorem that a particular loss, optimizer, or neural architecture will learn an optimal summary from finite data.

The supplementary code checks formulas and finite constructions; no trained-model gain is claimed. Appendices~\ref{app:complete}--\ref{app:fusion} develop Complete Attention and its observable coordinates. A fixed normalized fusion action can itself hide an entire precision matrix needed for the next observation (Proposition~\ref{prop:ca_fusionhidden}).

%% file: sections/09_discussion.tex
\section{Relationship to existing theory and architectural claims}\label{sec:discussion}
\paragraph{Future distinguishability is established.}
Predictive-state representations define state through predictions of future tests \citep{littman2001}. Statistical sufficiency and distinguishability arguments likewise predate this work. We do not claim that a present point estimate has historically been regarded as universally sufficient. The results calculate exactly which evidence a single \emph{native} observation reveals and how much remains invisible when the current predictor is fixed. For smooth residual families, the unweighted tangent construction and balancing argument remove observation weights and arbitrary terminal probing losses from that calculation.

\paragraph{Polynomial closure has a classical function-space ancestry.}
Finite-dimensional translation-invariant spaces and location--scale sufficiency already connect exact summaries to rigid function classes \citep{anselone1964,mattner2000}. Our classification uses that structure, rather than advertising a new classification of translation-invariant spaces. Its operational content is exact state minimality against continuous nonlinear encoders, ordinary observations, and the sharp $r-d$ equal-prediction count. Approximate statistics and regret-sufficient statistics address distinct goals \citep{huggins2017,foster2018,gopalan2024}.

\paragraph{Matrix evidence and test-time regression are not new algorithms.}
RLS, MesaNet, importance-state memories, and curvature-aware recurrences retain evidence beyond a current weight matrix \citep{vonoswald2026,bonnet2026,tumma2026}. Test-time regression unifies associative-memory sequence models \citep{wang2026}. The ridge theorem establishes necessity for exact cumulative ridge, not a newly discovered implementation. Conversely, changing the loss changes the closure class: predictor-plus-covariance is not an exact general-purpose nonlinear learner.

\paragraph{Current context compilation does not establish update closure.}
Context-to-parameter equivalences can reproduce a contextual computation with suitable parameter modifications \citep{goldwaser2026}. Our question is whether a compressed prefix permits the \emph{same subsequent observation} to update that computation correctly without revisiting the prefix. Likewise, sequence-model expressivity and efficiency tradeoffs concern explicitly specified model classes and tasks \citep{cooper2026}. Neither result licenses a universal claim against unrestricted Transformers, whose token states may retain all relevant history.

\paragraph{Exact dimension and finite precision are separate resources.}
Discontinuous infinite-precision encodings evade continuous dimension bounds. Theorem~\ref{thm:bits} allows arbitrary finite codes and gives a matching error exponent for a specified random one-update task. Its upper quantizer is not an efficient online algorithm. Streaming learning and attention bounds use other observation, precision, and error models \citep{brown2022,chen2026}. Transferring their conclusions, or ours, to typical language-model loss would require an additional reduction.

\paragraph{Scope.}
Static neural parameters, arithmetic cost, and history-dependent memory are different resources. Replay and a raw KV cache retain history rather than contradict the bounds. Learned feature changes, forgetting, finite dictionaries, restricted queries, and approximation can alter the computation being represented. The unrestricted-feature classification and bounded logistic theorem address different domains. No neural optimization theorem, downstream improvement, or universal practical inadequacy of fixed-memory models follows. Appendices~\ref{app:coherence}--\ref{app:fusion} develop the operator construction and distinguish response coordinates, execution width, and retained evidence.

\paragraph{Conclusion.}
The minimal exact state of a learner is governed by its continuation-identifiable evidence and not the size of its current predictor. Ordinary residual observations give a polynomial/nonpolynomial closure dichotomy and an exact count of evidence hidden at a fixed predictor. Ridge regression makes this information quantitative at finite precision under random future queries. Attention exposes the same distinction through its numerator and normalization mass. A memory intended to keep learning must preserve what changes its future response, including evidence that has no effect on any prediction it makes now.

%% file: statements.tex
\subsection*{AI use statement}

In this work, we used generative AI tools to help develop theoretical models or conceptual frameworks, formulate mathematical claims, assist in the writing of proofs, propose or refine hypotheses. 
We have not used generative AI tools to provide critical ingredients for proving mathematical claims
and generate synthetic data sets, design or provide feedback on research  methodology or experiments, implement methods, assist with translation, clean and reformat dataset, support qualitative and thematic data analysis, interpret results are not applicable to this work. 
Additionally, we used generative AI tools to summarize or analyse existing literature, discover research topics or identify gaps. 
We have reviewed all AI-assisted work. We proved every theorem, checked assumptions, checked references and we checked LLM-generated research ideas for potential plagiarism through a manual literature survey. 
We take responsibility for the final content of this work, including text, claims or artifacts produced with the aid of generative AI.

\subsection*{Ethics statement}
This work studies mathematical models of learning and information retention. It uses no personal data or human participants.

\subsection*{Reproducibility statement}
Assumptions and main conclusions are stated in Sections~\ref{sec:ridge}--\ref{sec:approx}; complete proofs are in Appendices~\ref{app:native}--\ref{app:approx}. Appendices~\ref{app:coherence}--\ref{app:fusion} give the supplementary operator constructions and their proofs.

%% file: appendices/01_native.tex
\section{Native identification and unweighted state: complete proofs}\label{app:native}
All exact lower bounds below permit an arbitrary decoder and require only continuity of the complete prefix encoder at each fixed length. The clock and fixed parameters are allowed; every history-dependent auxiliary array is counted. A restricted-family prefix summary need not be a recurrence of the same size on histories outside that family. The finite-bit result uses a separate measurable-code model.

\subsection{The classical dimension obstruction}
\begin{lemma}\label{lem:dimension}
Let $\Theta\subset\R^r$ be nonempty and open. Let $\theta\mapsto h_\theta$ be a continuous fixed-length history family such that every distinct pair is distinguished by an allowed continuation and readout. Every continuous exact encoder into $\R^m$ has $m\ge r$.
\end{lemma}
\begin{proof}
The composition $\theta\mapsto s(h_\theta)$ is injective, since equal states must give equal outputs on the same continuation and query. If $m<r$, compose it with the coordinate inclusion $\R^m\hookrightarrow\R^r$. Invariance of domain says the image of this continuous injection of an open subset of $\R^r$ is open. But it lies in a proper linear subspace, whose interior is empty. This contradiction proves the claim. Decoder continuity is not needed. Invariance of domain is the classical topological input \citep{hatcher2002}.
\end{proof}

\subsection{Existence, smooth dependence, and native identification}
Each $F_h$ in \eqref{eq:objective} is $\lambda$-strongly convex. Each differentiable convex summand has an affine lower bound; combining their finite sum with the quadratic lower bound from $R$ proves coercivity and existence of a unique minimizer. The Hessian satisfies $\nabla^2F_h\succeq\lambda I$. At fixed length, the implicit-function theorem applied to $\nabla F_h(w)=0$ gives local smooth dependence on all observations. Uniqueness makes the local maps agree, giving a smooth minimizer map on the whole observation domain.

Choose $t_0$ with $c=\rho'(t_0)\ne0$. For any history $h$ and point $w$, define the ordinary observation in \eqref{eq:steer}. Its gradient at $w$ is $x\rho'(x^\top w-y)=xc=-\nabla F_h(w)$. Thus the continued objective has unique global minimum $w$. If another history $h'$ has identical one-example responses, the same observation makes its optimum $w$ too. Subtracting stationarity equations gives $\nabla F_{h'}(w)=\nabla F_h(w)$. Since $w$ is arbitrary, the profiles agree as functions. Conversely, equal gradients remain equal after every common future summand and hence give the same present and continued optima. The fixed $R$ cancels in comparisons.

The distinguishing observation is an ordinary external input. Its existence may depend on the histories in the proof; the decoder is not given a gradient, desired optimum, or history-dependent oracle. Both compared states receive exactly the same observation.

\subsection{Finite evaluation coordinates}
\begin{lemma}\label{lem:evaluation}
If $V$ is an $r$-dimensional space of functions from a set $X$ to $\R^p$, there are $r$ scalar evaluations $L_j(f)=f_{c_j}(x_j)$ such that $L:V\to\R^r$ is an isomorphism. The space $V$ is closed under pointwise limits that exist everywhere. A pointwise smooth family in $V$ has smooth coefficients in any fixed basis.
\end{lemma}
\begin{proof}
Scalar evaluations separate $V$: a function whose every coordinate vanishes at every point is zero. Their span in $V^*$ has zero annihilator and is therefore all of $V^*$. Select a basis of $r$ evaluations. If $f_1,\ldots,f_r$ is any basis of $V$, the matrix $[L_if_j]$ is invertible. For any pointwise-convergent sequence in $V$, evaluating at these coordinates gives convergence of its coefficient vectors. The resulting linear combination agrees with the pointwise limit everywhere. Smoothness of coefficients follows from the same fixed inverse matrix applied to the evaluation functions.
\end{proof}

\subsection{Observation tangents replace positive sample weights}
\begin{lemma}[Unweighted tangent chart]\label{lem:tangent}
Let $Z\subset\R^k$ be open and path-connected, and let $g:Z\to(\R^p)^X$ be pointwise $C^1$. Suppose $g(z_*)=0$ for some $z_*$. Define
\[
 G=\spanop\{g(z):z\in Z\},\qquad
 T=\spanop\{Dg(z)[v]:z\in Z,v\in\R^k\}.
\]
If $G$ is finite dimensional then $T=G$. If $G$ is infinite dimensional then $T$ is infinite dimensional. For any integer $n\le\dim T$, some fixed-length unweighted history family with $n$ parameters has a locally injective accumulated function and $n$ independent evaluation derivatives.
\end{lemma}
\begin{proof}
If $G$ is finite dimensional, difference quotients lie in $G$, so their pointwise limits do too by Lemma~\ref{lem:evaluation}. Thus $T\subset G$.

If $T$ is finite dimensional, join $z_*$ to any $z$ by a piecewise $C^1$ path $\gamma$ inside $Z$. Such paths exist in connected open Euclidean domains. The derivative of $g(\gamma(t))$ belongs to $T$. Choose finite evaluation coordinates on $T$; their coefficients form an ordinary piecewise continuous finite-dimensional integrand. Integrate those coefficients. The fundamental theorem of calculus at each point of $X$ and each output coordinate identifies the resulting element of $T$ with $g(z)-g(z_*)=g(z)$. Hence $G\subset T$. This proves equality in the finite case, and proves the infinite case by contraposition.

Choose $n$ independent functions $v_j^*=Dg(z_j)[u_j]$. Choose $n$ scalar evaluations $L_i$ making $[L_iv_j^*]$ invertible. For $\theta$ near zero use the ordinary history
\[
 h_\theta=(z_1+\theta_1u_1,\ldots,z_n+\theta_nu_n).
\]
All observations stay in $Z$ on a sufficiently small neighborhood. The Jacobian at zero of $\theta\mapsto L\sum_jg(z_j+\theta_ju_j)$ is the invertible matrix just selected. The inverse-function theorem gives local injectivity of these evaluated sums and therefore of the accumulated functions. Each observation has coefficient one; only its location changes.
\end{proof}
For residual losses, $Z=\R^{d+1}$ and $g(x,y)(w)=x\rho'(x^\top w-y)$ is pointwise smooth with $g(0,y)=0$. If $r=\dim\G_\rho<\infty$, the tangent lemma gives an open $r$-parameter family of length $r$. Native identification and Lemma~\ref{lem:dimension} give $m\ge r$. For an infinite span, the same argument yields every finite lower bound.

For the upper bound, choose a basis $g_1,\ldots,g_r$ and write $g_z=\sum_j a_j(z)g_j$. Finite evaluations show that $a(z)$ is smooth. Store $s_h=\sum_{z\in h}a(z)$ and update by addition. The decoder uses the unique optimum with gradient $\nabla R+\sum_j(s_h)_jg_j$. It only needs to be correct on reachable states; an arbitrary extension elsewhere is harmless. This proves Theorem~\ref{thm:native} completely.

\subsection{A generalization that requires an extra probing hypothesis}
For arbitrary convex losses $\ell_z$, native observations need not steer an optimizer to every point. One must not silently infer identification from convexity alone. A sufficient, stronger readout family is
\[
 a\mapsto\argmin_w\{F_h(w)+(\mu/2)\norm{w-a}^2\},\qquad\mu>0.
\]
For any $w$, choose $a=w+\nabla F_h(w)/\mu$. Equality of these proximal responses identifies the old gradient at $w$. With positive sample weights, independent loss gradients then yield an open coefficient chart. Without weights, the tangent lemma requires its stated domain and zero-observation conditions. The residual theorem needs neither these terminal quadratic probes nor positive history weights: its native observations provide both identification and the unweighted chart.

%% file: appendices/02_closure.tex
\section{Polynomial closure and fixed-predictor memory}\label{app:poly}
\subsection{Finite affine orbits force a polynomial}
Suppose $\G_\rho$ is finite dimensional. For $x=ae_1$, $a>0$, $y=-b$, and $w=te_1$, the first gradient coordinate is $a\rho'(at+b)$. Dividing by $a$ shows that
\[
 V=\spanop\{t\mapsto\rho'(at+b):a>0,b\in\R\}
\]
is finite dimensional. It is invariant under all translations and positive dilations, including their inverses. Pointwise closure from Lemma~\ref{lem:evaluation} applied to difference quotients shows that differentiation $\mathsf D:f\mapsto f'$ is an endomorphism of $V$. The chain rule gives $\mathsf DT_a=aT_a\mathsf D$, hence $T_a^{-1}\mathsf DT_a=a\mathsf D$.

After complexification, the finite spectrum of $\mathsf D$ is invariant under multiplication by every positive real number. It can contain only zero. For $\ell=\dim V$, Cayley--Hamilton gives $\mathsf D^\ell=0$, so every function in $V$ is a polynomial of degree at most $\ell-1$. In particular $\rho'$ and then $\rho$ are polynomial. No analyticity of the original loss is assumed; smoothness suffices. This is an application of classical finite-dimensional function-space reasoning \citep{anselone1964,mattner2000}.

\subsection{Exact dimension for polynomial losses}
If $\rho$ has degree $D\ge1$, its translation span is finite dimensional and differentiation invariant. It contains $\rho,\rho',\ldots,\rho^{(D)}$, whose distinct degrees and nonzero leading coefficients make them a basis for univariate polynomials of degree at most $D$. Thus for each $x$, translates $w\mapsto\rho(x^\top w-y)$ span $(x^\top w)^j$ for $0\le j\le D$.

For each $j$, these powers as $x$ varies span all degree-$j$ homogeneous polynomials in $w$. Indeed,
\[
 (x^\top w)^j=\sum_{|\alpha|=j}\frac{j!}{\alpha!}x^\alpha w^\alpha.
\]
A coefficient functional annihilating all such powers yields a polynomial in $x$ vanishing everywhere, hence has all coefficients zero. The residual-loss span is therefore the whole space of multivariate polynomials of total degree at most $D$, dimension $\binom{d+D}{D}$. Taking gradients removes exactly the constants and gives $r_D(d)=\binom{d+D}{D}-1$.

Theorem~\ref{thm:native} gives exact unweighted minimality. Explicit attainment stores all nonconstant coefficients of the accumulated loss polynomial; each example adds its coefficients. The readout minimizes the resulting strongly convex objective. This is a state-dimension claim, not a dimension-independent solver-complexity claim.

\subsection{Present readout dimension}
The smooth map $h\mapsto w_h$ is a $d$-coordinate read-only state. Every $w\in\R^d$ is continuously reached by the single observation \eqref{eq:steer} applied to the empty objective $R$. Distinct $w$ are identified by basis queries. Lemma~\ref{lem:dimension} gives the matching read-only lower bound $d$.

\subsection{Balancing preserves all but at most \texorpdfstring{$d$}{d} profile dimensions}
\begin{lemma}[Fixed-optimum chart]\label{lem:balance}
Suppose an unweighted smooth history family $h_\theta$, with $\theta$ near zero in $\R^n$, has accumulated gradient profile $G_\theta$ whose parameter derivative at zero is injective as a map into functions. For every $w_0$, a family with one additional ordinary residual observation has optimum $w_0$ throughout and profile derivative of rank at least $n-d$. When $n>d$, it contains a distinguishable local $(n-d)$-parameter subfamily.
\end{lemma}
\begin{proof}
Put $u_\theta=\nabla R(w_0)+G_\theta(w_0)$ and fix $c=\rho'(t_0)\ne0$. Append
\[
 z_\theta=\left(-u_\theta/c,\ (-u_\theta/c)^\top w_0-t_0\right).
\]
The resulting full objective has zero gradient at $w_0$, which is its unique minimum by strong convexity. Its profile is $\widetilde G_\theta=G_\theta+g_{z_\theta}$. The derivative of the balancing profile factors through $D_\theta u_0:\R^n\to\R^d$, and therefore has rank at most $d$. Consequently
\[
 \rank D_\theta\widetilde G_0\ge\rank D_\theta G_0-\rank D_\theta(g_{z_\theta})_0\ge n-d.
\]
The domain of these linear maps is finite dimensional, so the rank inequality needs no topological structure on the function space. Choose $n-d$ independent derivative images and point evaluations making their matrix nonsingular. Restrict parameters to the corresponding linear subspace. The inverse-function theorem gives an open chart of different profiles, which native observations distinguish.
\end{proof}
For finite $r$, use the $n=r$ chart from the tangent lemma and then the dimension obstruction to obtain the lower bound $r-d$. For an infinite span, choose $n$ arbitrarily large, proving an unbounded requirement even at one fixed predictor.

For a polynomial loss, evaluation at $w_0$ is a surjective map $E_{w_0}:\G_\rho\to\R^d$, because constant vector gradients belong to $\G_\rho$. Profiles of histories with optimum $w_0$ lie in the affine space
\[
 \{G\in\G_\rho:E_{w_0}G=-\nabla R(w_0)\},
\]
of dimension $r-d$. Affine coordinates on this space depend continuously on the history and reconstruct the entire profile, so they support any supplied continuation. This proves the restricted-prefix upper bound. A recurrent implementation can expand this state to the full $r$ coefficients when it leaves the family; the result does not give a global recurrence of size $r-d$.

Together these arguments prove Theorem~\ref{thm:poly}. A constant loss changes only an objective constant and needs no state. The steering argument is not applied in that degenerate case.

\subsection{Why the domain assumptions matter}
If scalar features are fixed at $x=1$ and $\rho(t)=e^t$, the gradient profile is $e^w\sum_i e^{-y_i}$ and one coordinate suffices. Arbitrarily varying feature scales is what supplies dilation invariance and excludes this counterexample. Discrete feature/label dictionaries can also have finite spans for nonpolynomial losses. Classical statistical sufficiency concerns representations of a likelihood; native one-example identification is the bridge from the optimizer question here to the full gradient function. This is not a claim that polynomial sufficiency itself is new.

%% file: appendices/03_ridge.tex
\section{Ridge regression and normalized keys}\label{app:ridge}
\subsection{Continuous reachability and exact state minimality}
For any $C\succ\lambda I$ and $M\in\R^{p\times d}$, set $X=(C-\lambda I)^{1/2}$ and $Y=MX^{-\top}$. These depend continuously on $(C,M)$. Their $d$ columns as observations give $\lambda I+XX^\top=C$ and $YX^\top=M$. This is an open $(s_d+pd)$-dimensional parameter domain. Fixing $C$ allows all $W$ via $M=WC$, proving read-only minimum $pd$ by basis queries and Lemma~\ref{lem:dimension}.

Adding $(k,y)$ yields $(C+kk^\top,M+yk^\top)$. Direct multiplication verifies
\[
 (C+kk^\top)^{-1}=P-\frac{Pkk^\top P}{1+k^\top Pk},\qquad P=C^{-1}.
\]
Substitution gives \eqref{eq:gain}. Equal states supporting present readouts have the same $W$. Appending the common label $y=Wk+e_1$ identifies $\Phi_P(k)$ from the first output row at basis queries. Since $1-k^\top\Phi_P(k)=1/(1+k^\top Pk)>0$, gain inversion identifies $Pk$. Basis keys identify $P$ and then $M=WP^{-1}$. The reachable open chart proves the full lower bound. Additive $(C,M)$ storage attains it for every future observation. Equivalently, store $(P,W)$ with the classical RLS recurrence
\[
 P^+=P-\frac{Pkk^\top P}{1+k^\top Pk},\qquad
 W^+=W+(y-Wk)\frac{k^\top P}{1+k^\top Pk}.
\]
It has $s_d+pd$ independent entries and dense update cost $O(d^2+pd)$. On zero labels, $M=W=0$ while $C$ ranges over an open $s_d$-dimensional set. The fixed label $e_1$ suffices to identify $C$; storing it attains the restricted-prefix bound. This proves Theorem~\ref{thm:ridge}.

\subsection{Unit keys and known length}\label{app:unit}
Fix $N\ge d$, $d\ge2$. Choose $d$ independent unit vectors $u_i$ with every pairwise inner product nonzero; normalized columns of $I+\varepsilon\mathbf1\mathbf1^\top$ for small $\varepsilon>0$ are an example. Fix the other $N-d$ keys. The derivative of the Gram map from the first $d$ spheres is
\[
 (\dot u_i)\mapsto\sum_i(u_i\dot u_i^\top+\dot u_i u_i^\top),\qquad u_i^\top\dot u_i=0.
\]
Its image is trace-zero. If a symmetric matrix $H$ annihilates this image, then $Hu_i$ is parallel to $u_i$ for every $i$. Symmetry and nonzero pairwise inner products force all associated eigenvalues equal. The $u_i$ span, so $H$ is scalar. Thus the derivative is onto the trace-zero symmetric space. The submersion theorem gives a local smooth section over an open part of a fixed-trace Gram slice of dimension $s_d-1$.

The first $d$ keys remain independent nearby; their labels realize arbitrary $M$ continuously, setting the remaining labels zero. This gives a distinguishable chart of dimension $s_d-1+pd$. Future unit keys include the coordinate basis, so the identification proof applies. For attainment, omit one diagonal entry of $C$ and recover it from the known trace $\lambda d+N$, while storing $M$. After an update the trace is $\lambda d+N+1$. The exact dimension is $pd+s_d-1$. No formula is asserted for $N<d$, where rank constraints can further change the attainable set.

\subsection{Exact collision and a restricted delta recurrence}
For \eqref{eq:collision}, the eigenvalues of $C_\pm-\lambda I$ are $1/4,5/4$, so both histories are attainable with two zero-label observations. Direct inversion gives
\[
 (C_\pm+e_1e_1^\top)^{-1}e_1=(4/7,\mp2/7)^\top.
\]
At the continued query $e_2$, any common predicted value $z$ has equally weighted squared error
\[
 \tfrac12(z+2/7)^2+\tfrac12(z-2/7)^2=z^2+4/49\ge4/49.
\]
The worst error is at least the average. Randomized common outputs satisfy the same bound in expectation. A raw recurrence $S^+=S+\beta(y-Sk)k^\top$ stays at zero on zero labels from zero initialization, irrespective of the scalar gate. A broader model can avoid the collision by another evidence-carrying path; that path is included in its state budget rather than excluded by a model name.

%% file: appendices/04_logistic.tex
\section{Bounded logistic responses and equal finite jets}\label{app:logistic}
A paired observation contributes
\[
 \log(1+e^{-aw})+\log(1+e^{aw})=2\log(2\cosh(aw/2)).
\]
Its derivative is $a\tanh(aw/2)$, so \eqref{eq:logistic_profile} follows. The derivative is odd, and the objective is strongly convex, making zero the unique optimum for every prefix in the family.

For a continued positive-label observation $(b,+1)$, stationarity is \eqref{eq:logistic_update}. Define
\[
 H_a(w,b)=F_a'(w)-\frac{b}{1+e^{bw}}.
\]
This is real analytic near $(0,0)$, vanishes there, and has $\partial_wH_a(0,0)=\lambda+\frac12\sum_i a_i^2>0$, while $\partial_bH_a(0,0)=-1/2$. The analytic implicit-function theorem gives \eqref{eq:logistic_update} and a strictly positive derivative $w_a'(0)$. Thus on some $(0,\varepsilon)$ its image contains an open interval of positive $w$.

Suppose two histories give the same $w_a(b)$ for all $0<b<b_0$. Subtracting their stationarity equations shows their derivatives $F_a'$ agree on an open interval. Real analyticity on all of $\R$ implies equality everywhere. Letting $w\to\infty$ first identifies $\sum_i a_i$. The remaining tail in \eqref{eq:logistic_tail} is a positive sum, with
\[
 -\lim_{w\to\infty}\frac1w\log\sum_i\frac{2a_i}{1+e^{a_iw}}=a_1.
\]
The two histories have the same smallest scale. Subtract its entire function $a_1\tanh(a_1w/2)$ from the equal profiles. Iterate to identify all ordered scales. Positivity prevents cancellation of the leading exponential. Distinctness and the known number of paired observations ensure exact recovery of the ordered tuple. The open set $\{1<a_1<\cdots<a_N<2\}$ therefore maps injectively into any exact summary. Lemma~\ref{lem:dimension} proves Theorem~\ref{thm:logistic}.

For the finite-jet claim, write
\[
 T_K(a)=\left(\sum_i a_i^2,\sum_i a_i^4,\ldots,\sum_i a_i^{2K}\right).
\]
When $K\ge1$ and $N>K$, its Jacobian has entries $(DT_K)_{ji}=2j a_i^{2j-1}$. After rescaling rows and columns, any $K$ columns form a Vandermonde matrix in distinct positive numbers $a_i^2$. Thus its rank is $K$ throughout the ordered domain. The regular level-set theorem gives a local smooth $(N-K)$-dimensional manifold through each point, with a coordinate chart on an open subset of $\R^{N-K}$.

The Taylor expansion of $2\log(2\cosh(aw/2))$ at zero contains only even powers; its coefficient at degree $2j$ is a scalar multiple of $a^{2j}$. The fixed regularizer contributes only at degree two. Consequently every derivative of orders $0,\ldots,2K+1$ is identical on the level set: the constant is $2N\log2$, the odd derivatives vanish, and the even derivatives are fixed by $T_K$. For $K=0$, the whole ordered domain already has identical value and first derivative at zero, so it is the required $N$-dimensional family. In all cases, the preceding identification argument distinguishes every pair by a bounded native continuation. The dimension obstruction proves Theorem~\ref{thm:jet}.

No quantitative separation is inferred from this proof. Near-colliding power sums can produce extremely small differences in a restricted response interval. The average finite-bit result in the next appendix is proved independently in a uniformly conditioned ridge model.

%% file: appendices/05_bits.tex
\section{Finite-bit memory: complete average-error proof}\label{app:bits}
Let $s=s_d=d(d+1)/2$, and use Frobenius-orthonormal coordinates on $\Sym_d$. All volumes refer to the resulting $s$-dimensional Lebesgue measure. Fix $a>18r>0$ and $b=a+18r<1/\lambda$ as in \eqref{eq:body}.

\subsection{Uniformly conditioned reachable zero-prediction histories}
Every $P=aI+H$ in the specified family satisfies
\[
 (a-18r)I\preceq P\preceq bI\prec\lambda^{-1}I.
\]
Thus $C=P^{-1}\succ\lambda I$, and $X=(P^{-1}-\lambda I)^{1/2}$ with zero labels realizes $C$ using exactly $d$ observations. This is a continuous encoding of the full-dimensional matrix family by native prefixes. The condition number of $C$ is at most $(a+18r)/(a-18r)$, and every column of $X$ has norm at most $\sqrt{(a-18r)^{-1}-\lambda}$, independent of $d$ for fixed $a,r,\lambda$. The current predictor is identically zero. The ordinary next observation $(k,1)$ yields $f_P(k,q)=q^\top Pk/(1+k^\top Pk)$.

\subsection{A uniformly conditioned response metric}
\begin{lemma}\label{lem:gainmetric}
For symmetric positive $P,Q$ with operator norms at most $b$, and independent random unit vectors with $\E kk^\top=\E qq^\top=I/d$, inequality \eqref{eq:gain_distance} holds.
\end{lemma}
\begin{proof}
For a fixed unit $k$, set $u=Pk/(1+k^\top Pk)$ and $v=Qk/(1+k^\top Qk)$. Gain inversion gives the exact identity
\[
 (P-Q)k=\frac{u-v+Qk\,k^\top(u-v)}{1-k^\top u}.
\]
Since $(1-k^\top u)^{-1}=1+k^\top Pk\le1+b$ and $\norm{Qk}\le b$, this yields
\[
 \norm{(P-Q)k}\le(1+b)^2\norm{u-v}.
\]
Conversely, subtract the forward gain formulas:
\[
 u-v=\frac{(P-Q)k}{1+k^\top Pk}
 -\frac{Qk\,k^\top(P-Q)k}{(1+k^\top Pk)(1+k^\top Qk)}.
\]
All denominators are at least one, so $\norm{u-v}\le(1+b)\norm{(P-Q)k}$. Independence and isotropy give
\[
 \E_q[q^\top(u-v)]^2=\norm{u-v}^2/d,
 \qquad\E_k\norm{(P-Q)k}^2=\norm{P-Q}_\F^2/d.
\]
Combining proves the lemma. The argument applies both to uniform sphere probes and independent uniform coordinate-basis probes.
\end{proof}

\subsection{At least half of the Frobenius ball is well conditioned}
\begin{lemma}\label{lem:volume}
For every $d\ge2$, with $v_s$ the unit ball volume in $\R^s$,
\[
 \operatorname{vol}(\cK_d)\ge\tfrac12 v_s(r\sqrt d)^s.
\]
\end{lemma}
\begin{proof}
Let $G$ be a standard Gaussian in Frobenius-orthonormal coordinates of $\Sym_d$. Diagonal entries have variance one, off-diagonal entries variance $1/2$, and $u^\top Gu$ is standard normal for any unit $u$. A $1/4$-net of the unit sphere with at most $9^d$ points exists: a maximal separated set has disjoint radius-$1/8$ balls inside a radius-$9/8$ ball. For a symmetric matrix, the net satisfies $\norm G_\op\le2\max_{u\text{ in net}}|u^\top Gu|$, by approximating a maximizing unit vector and bounding the two error terms. Gaussian tails and a union bound give
\[
 \Pr\{\norm G_\op>6\sqrt d\}\le 2\exp[-(4.5-\log9)d].
\]
Also $\norm G_\F^2$ has the $\chi_s^2$ law. Exponential Markov with parameter $3/2$ yields
\[
 \Pr\{\norm G_\F^2<s/4\}
 \le\exp[-s(\log4-3/4)/2].
\]
At $d=2,s=3$, the sum of the two bounds is less than $0.406$; both decrease with $d$. With probability greater than one half, therefore,
\[
 \frac{\norm G_\op}{\norm G_\F}\le12\sqrt{d/s}.
\]
The direction $U=G/\norm G_\F$ is uniform on the Frobenius unit sphere. At least half the directions satisfy this ratio bound. Every point $tU$ in one of those directions with $0\le t\le r\sqrt d$ has
\[
 \norm{tU}_\op\le12rd/\sqrt s\le12\sqrt2\,r<18r.
\]
The entire corresponding radial segment belongs to $\cK_d$. Radial integration shows that these directions account for at least half the ball volume.
\end{proof}

\subsection{A second-moment inequality for arbitrary randomized finite codes}
\begin{lemma}\label{lem:quantization}
Let $X$ be uniform on a measurable body $K\subset\R^s$ of finite positive volume $V$ and finite second moment. Let $S$ take at most $M$ values, using any randomized measurable encoder. Then
\begin{equation}\label{eq:quantization}
 \E\norm{X-\E[X\mid S]}^2
 \ge\frac{s}{s+2}\left(\frac{V}{v_s}\right)^{2/s}M^{-2/s}.
\end{equation}
\end{lemma}
\begin{proof}
Let $\pi_j(x)=\Pr(S=j\mid X=x)$, extended by zero outside $K$. It lies in $[0,1]$, and has integral $m_j=Vp_j$, with $p_j=\Pr(S=j)$. For any center $c$ and any measurable $0\le\pi\le1$ of mass $m$, layer-cake integration gives
\begin{align*}
 \int\norm{x-c}^2\pi(x)\,dx
 &=\int_0^\infty\int_{\norm{x-c}^2>t}\pi(x)\,dx\,dt\\
 &\ge\int_0^\infty(m-v_st^{s/2})_+\,dt
 =\frac{s}{s+2}v_s^{-2/s}m^{1+2/s}.
\end{align*}
The inequality uses only the unit density cap, since a ball of squared radius $t$ holds mass at most $v_st^{s/2}$. Apply it to each $\pi_j$ with $c_j=\E[X\mid S=j]$, omitting zero-probability codes. Sum and divide by $V$:
\[
 \E\norm{X-\E[X\mid S]}^2
 \ge\frac{s}{s+2}(V/v_s)^{2/s}\sum_j p_j^{1+2/s}.
\]
Convexity, padding with zero masses up to $M$ entries, gives $\sum_jp_j^{1+2/s}\ge M^{-2/s}$. This proves the claim for arbitrary randomized encoders, not just nearest-neighbor or deterministic partitions.
\end{proof}
Combining Lemmas~\ref{lem:volume} and \ref{lem:quantization}, with $X=H$, gives
\begin{equation}\label{eq:matrix_distortion}
 \E\norm{P-\E[P\mid S]}_\F^2
 \ge\frac{s}{s+2}2^{-2/s}r^2d\,2^{-2B/s}
 \ge\frac{r^2d}{4}\,2^{-2B/s},
\end{equation}
since $s\ge3$. The last bound uses $s/(s+2)\ge3/5$ and $2^{-2/s}\ge2^{-2/3}$, whose product exceeds $1/4$.

\subsection{Why the lower bound applies to every decoder}
Condition on the code $S$. The future probes are independent of the prefix and encoding randomness. For squared error the optimal prediction is $\E[f_P(k,q)\mid S,k,q]$, whether or not this function is any ridge predictor. Let $P,P'$ be independent posterior copies given $S$. The conditional variance identity yields
\begin{align*}
 \inf_{\widehat f}\E[\widehat f(S,k,q)-f_P(k,q)]^2
 &=\tfrac12\E[f_P(k,q)-f_{P'}(k,q)]^2\\
 &\ge\frac{\E\norm{P-P'}_\F^2}{2d^2(1+b)^4}
 =\frac{\E\norm{P-\E[P\mid S]}_\F^2}{d^2(1+b)^4}.
\end{align*}
Use \eqref{eq:matrix_distortion} to obtain the lower half of \eqref{eq:bits}. Extra independent decoder randomness cannot improve squared-error risk over its mean. Private encoder randomness is already allowed in Lemma~\ref{lem:quantization}. Independent shared randomness $U$ is handled by conditioning on $U$: before encoding the prefix remains uniform, and the same argument applies for each seed value with at most $2^B$ codes. Averaging gives the same lower bound.

\subsection{A covering quantizer attains the exponent}
Let $R_0=r\sqrt d$. If $B\ge s$, put $\delta=2R_0/(2^{B/s}-1)$. A maximal $\delta$-separated subset of the compact body $\cK_d$ is a $\delta$-cover. Disjoint balls of radius $\delta/2$ around its centers fit inside a ball of radius $R_0+\delta/2$, so its cardinality is at most
\[
 (1+2R_0/\delta)^s=2^B.
\]
A deterministic nearest-center code gives $\norm{P-\widehat P}_\F\le\delta\le4R_0\,2^{-B/s}$. Use the coherent decoder $f_{\widehat P}$ and the upper half of Lemma~\ref{lem:gainmetric} to obtain
\[
 \E(f_P-f_{\widehat P})^2\le\frac{16(1+b)^2r^2}{d}2^{-2B/s}.
\]
For $0\le B<s$, one center $aI$ gives risk at most $(1+b)^2r^2/d$, which is no larger than the displayed upper bound since $16\,2^{-2B/s}>4$. Both cases complete Theorem~\ref{thm:bits}.

This is an offline coding upper bound for a prefix statistic. The decoder is inexpensive given its center, but finding a codebook, encoding a matrix, or maintaining such a code through a stream can be computationally difficult. No online quantization or efficient training guarantee is implicit.

\subsection{Supplementary worst-case finite-bit bound}
Let $P=aI+H$ range over the full operator-norm ball $\norm H_\op\le r$ with $a>r$ and $a+r\le b$. Suppose a deterministic code with at most $2^B$ states returns scalar responses within $\varepsilon$ uniformly over all unit $k,q$ for every such $P$. Two matrices in one code cell have response difference at most $2\varepsilon$. Taking the supremum over $q$ and using pointwise gain inversion shows
\[
 \norm{P-Q}_\op\le2(1+b)^2\varepsilon.
\]
Each nonempty cell is contained in an operator-norm ball of that radius around one representative. Volume comparison in $s$ dimensions gives
\[
 2^B\,[2(1+b)^2\varepsilon]^s\ge r^s,
 \qquad
 B\ge s\log_2\frac{r}{2(1+b)^2\varepsilon}.
\]
When the logarithm is negative the bound is vacuous. This worst-case statement is not the randomized average statement of Theorem~\ref{thm:bits}; the latter needs the posterior and moment arguments above.

%% file: appendices/06_attention.tex
\section{Kernel attention and native unweighted key--value memory}\label{app:attention}
\subsection{A pair of native continuations identifies the evidence}
For a fixed appended key $k_*$ and query $q$, set $\alpha=\kappa(q,k_*)>0$. Appending weight-one value zero gives $u_0=B/(Z+\alpha)$, while appending $e_1$ gives $u_1=(B+\alpha e_1)/(Z+\alpha)$. Hence $\delta=(u_1-u_0)_1=\alpha/(Z+\alpha)>0$, and \eqref{eq:att_ident} recovers $Z$ and $B$. Equality after all single-token continuations therefore implies equality of these functions at every query. Conversely, equal $(B,Z)$ gives equal outputs after adding any common numerator and denominator terms, proving all-future closure. The two continuations are alternatives tested on the same old history; the second is not appended after the first.

\subsection{Finite kernel sections and weighted exact minimality}
Suppose $r=r_\kappa<\infty$ and choose a basis $\phi_1,\ldots,\phi_r$ of kernel sections. Write $\kappa(q,k)=\sum_j\phi_j(q)a_j(k)$. Scalar evaluations as in Lemma~\ref{lem:evaluation} express $a(k)$ as a fixed matrix times finitely many kernel values, so it is continuous under the stated pointwise continuity assumption. Store
\[
 z_h=\sum_i t_i a(k_i)\in\R^r,\qquad
 A_h=\sum_i t_i a(k_i)v_i^\top\in\R^{r\times p}.
\]
Their additive updates reconstruct $(Z,B)$ exactly. This attains $r(p+1)$ coordinates.

For necessity, choose keys $k_1,\ldots,k_r$ with independent sections, positive weights $t_j$ in open intervals, and independently varying vectors $\beta_j=t_jv_j\in\R^p$. The map from $(t,\beta)$ to the corresponding ordinary weighted observations is continuous, since $v_j=\beta_j/t_j$. The denominator identifies each $t_j$ and the numerator each $\beta_j$ by section independence. Proposition~\ref{prop:attention} and Lemma~\ref{lem:dimension} force $r(p+1)$ coordinates. If the section span is infinite, choose arbitrarily many independent keys and repeat. This proves the first part of Theorem~\ref{thm:attention}. The weighted-prefix assumption is retained; the proof is not a lower bound of that size for all unweighted fixed-length kernel models.

\subsection{Independence of real exponential sections}
\begin{lemma}\label{lem:exponentials}
For distinct $k_1,\ldots,k_m\in\R^d$, the functions $q\mapsto e^{q^\top k_i}$ are linearly independent on every nonempty open subset of $\R^d$.
\end{lemma}
\begin{proof}
Choose $q_0$ in the open set and a direction $v$ for which $v^\top(k_i-k_j)\ne0$ whenever $i\ne j$. Such a direction exists because finitely many proper hyperplanes cannot cover $\R^d$. The line $q_0+tv$ stays in the query set for sufficiently small $|t|$. A vanishing linear combination there is
\[
 \sum_i c_i e^{q_0^\top k_i}e^{t v^\top k_i}=0.
\]
Differentiate at $t=0$ through orders $0,\ldots,m-1$. The resulting Vandermonde system in the distinct numbers $v^\top k_i$ is invertible. Hence all $c_i$ vanish.
\end{proof}

\subsection{Native unweighted exponential attention}
Fix $N$ disjoint open key balls. A history with one key in each ball and arbitrary $v_i\in\R^p$ varies over an open $N(d+p)$-dimensional domain. Suppose two such histories have identical one-token responses. Their $(B,Z)$ functions agree by Proposition~\ref{prop:attention}. Apply Lemma~\ref{lem:exponentials} to the union of their key sets in the equality
\[
 \sum_{i=1}^N e^{q^\top k_i}=\sum_{i=1}^N e^{q^\top k_i'}.
\]
The key multisets agree, and the disjoint balls fix matching indices. Apply the same independence to each output coordinate of $B$ to obtain $v_i=v_i'$. The dimension obstruction gives $N(d+p)$, attained by storing the original ordered keys and values at this length. On the zero-valued subfamily, only the $Nd$ key coordinates vary; the same denominator argument gives the lower bound and raw key storage attains it. Both statements are prefix bounds at fixed $N$, and imply no fixed-size continuous exact recurrent state across all $N$.

\subsection{Boundary cases that are not covered by the count}
Strict positivity ensures that the chosen continuation probes reveal the denominator at every query; kernels with zeros require an observability condition. An empty history has no old ratio but is handled by a known initial state. Repeated exponential keys merge by total value sums and multiplicities, so the distinct-key chart is required for the full count. A finite key dictionary permits a finite histogram state. A restricted finite query set can have finite kernel-section dimension. Arbitrary real encoders with no continuity or finite-precision condition are excluded from the exact theorem. None of these statements rules out approximate softmax memory on bounded domains or low-error tasks.

%% file: appendices/07_approx.tex
\section{Continuation-stable approximation and local curvature}\label{app:approx}
\subsection{Gradient evidence}
Let $w^+$ minimize $F+U$. Its gradient is zero. Strong convexity of $F$ and convexity of $U$ give
\[
 \lambda\norm{\widehat w-w^+}^2
 \le\inner{\nabla(F+U)(\widehat w)}{\widehat w-w^+}
 \le\norm{\nabla(F+U)(\widehat w)}\norm{\widehat w-w^+}.
\]
At $\widehat w\in\Omega$, the stated profile error and residual imply
\[
 \norm{\nabla(F+U)(\widehat w)}
 \le\norm{\nabla F(\widehat w)-\widehat G(\widehat w)}
 +\norm{\widehat G(\widehat w)+\nabla U(\widehat w)}\le\delta+\eta.
\]
If the displacement is zero the conclusion is immediate; otherwise divide by its norm to prove Proposition~\ref{prop:grad_approx}. The domain need not be convex for this argument, since strong monotonicity is global and the error is only evaluated at the returned point. Uniformity on $\Omega$ is a usable certificate before that point is known. The existence of a returned approximate root is a hypothesis, not a consequence of uniform approximation alone.

If $|p(t)-\rho'(t)|\le\varepsilon$ on a residual interval containing every $x_i^\top w-y_i$ for $w\in\Omega$, the triangle inequality proves \eqref{eq:moments}. If $p$ has degree $L$, each vector function $x_ip(x_i^\top w-y_i)$ is the gradient of a scalar polynomial of degree at most $L+1$; finitely many accumulated coefficients suffice. The bound scales with $\sum_i\norm{x_i}$, so an error tolerance for fixed prefix length cannot automatically be reused at all lengths. Neither a globally convex polynomial approximation nor existence of a particular numerical solver is asserted.

\subsection{Attention evidence}
Write errors $e_B=\widehat B-B$, $e_Z=\widehat Z-Z$, and let $(B_A,A)$ be the exact common continuation contributions. With $T=Z+A$,
\[
 \widehat f^+-f^+
 =\frac{e_B-f^+e_Z}{T+e_Z},\qquad
 f^+=\frac{B+B_A}{T}.
\]
Positive kernel weights make the exact output a convex combination of values, so $\norm{f^+}\le V$. Also $T+e_Z\ge T-\delta Z=T(1-\theta\delta)>0$, where $\theta=Z/T$. Thus
\[
 \norm{\widehat f^+-f^+}
 \le\frac{\eta Z+V\delta Z}{T(1-\theta\delta)}
 =\frac{\theta(\eta+\delta V)}{1-\theta\delta}.
\]
This proves Proposition~\ref{prop:att_approx}. The approximated denominator need not come from positive individual pseudo-weights, but it must satisfy the stated relative error. This is not a bound based only on the old normalized output.

\subsection{When predictor plus Hessian is a controlled approximation}
The exact logistic obstruction does not negate the local value of curvature. The following elementary estimate makes that scope explicit.
\begin{proposition}\label{prop:curvature}
Let $F$ be $\lambda$-strongly convex with minimizer $w_0$, and let $\ell$ be a differentiable convex appended loss. Put $w^+=\argmin(F+\ell)$, $v=w^+-w_0$, $H=\nabla^2F(w_0)$, and $g=\nabla\ell(w_0)$. Suppose along the segment joining $w_0$ and $w^+$, $\nabla^2F$ is operator-norm $L$-Lipschitz and $\nabla\ell$ is $L_\ell$-Lipschitz. Then
\[
 \norm v\le\norm g/\lambda,\qquad
 \norm{v+H^{-1}g}
 \le\frac{L\norm g^2}{2\lambda^3}+\frac{L_\ell\norm g}{\lambda^2}.
\]
\end{proposition}
\begin{proof}
Strong monotonicity of $\nabla(F+\ell)$ between $w_0$ and $w^+$ gives $\lambda\norm v\le\norm g$. Taylor integration along the segment yields $\nabla F(w_0+v)=Hv+r_F$ with $\norm{r_F}\le(L/2)\norm v^2$. Also $\nabla\ell(w_0+v)=g+r_\ell$ with $\norm{r_\ell}\le L_\ell\norm v$. Stationarity gives $v+H^{-1}g=-H^{-1}(r_F+r_\ell)$. Since $\norm{H^{-1}}_\op\le1/\lambda$, substitute the displacement bound to obtain the result.
\end{proof}
For an appended loss scaled by a small $t$, both $g$ and $L_\ell$ scale by $t$, so the error is $O(t^2)$ under fixed local constants. Such scaling is used only for this approximation statement; the exact main lower bounds use ordinary unit-weight observations. A useful first-order response approximation is not an exact finite-dimensional continuation state.

%% file: appendices/08_coherence.tex
\section{Supplement: static observation versus a reusable operator}\label{app:coherence}
The static counterpart of continuation sufficiency is whether a present observation of a shared operator supports an additional probe. We first state this criterion, then develop its positive-operator specialization. Appendices~\ref{app:complete}--\ref{app:fusion} give the associated Complete Attention construction, response geometry, execution results, and a fusion-memory application.

\subsection{A quotient that supports one read need not support the next}
Let $V,Y,Z$ be finite-dimensional real vector spaces and $L:V\to Y$, $M:V\to Z$ linear. The current observation is $Lv$ and a second required observation is $Mv$. There exists a decoder $D$ with $Mv=D(Lv)$ for every $v\in V$ if and only if
\[
 \ker L\subseteq\ker M.
\]
Necessity follows by comparing $v$ and $v+u$ for $u\in\ker L$. For sufficiency, define $D(Lv)=Mv$ on the image of $L$; the kernel condition makes this well defined. The decoder can be linear there.

On a nonempty open subset of $V$, a continuous state supporting both observations requires at least $\rank(L,M)$ coordinates. Choose a linear complement to $\ker(L,M)$ and an affine slice through an interior point. Its intersection with the open set contains an open neighborhood of dimension $\rank(L,M)$ distinguished by the joint readout, so Lemma~\ref{lem:dimension} applies. Storing the joint image attains this dimension. A minimal present quotient need not be reusable for another observation of the same action.

\subsection{Multiple probes for positive operator mixtures}
Let $W_i\in\Sym_d$ satisfy $W_i\succeq0$ and $\sum_iW_i=I$. In the positive-definite interior, a tangent is a tuple of symmetric $D_i$ with $\sum_iD_i=0$. Suppose source $i$ is probed by a matrix $X_i\in\R^{d\times b}$, centered so $\sum_iX_i=0$. The response tangent is
\[
 L_XD=\sum_iD_iX_i\in\R^{d\times b}.
\]
With Frobenius inner products, writing $\sym(A)=(A+A^\top)/2$, its adjoint and Gram operator are
\begin{align*}
 (L_X^*Z)_i&=\sym(ZX_i^\top),\\
 L_XL_X^*Z&=\frac12\left[Z\sum_iX_i^\top X_i+\sum_iX_iZ^\top X_i\right].
\end{align*}
To verify the adjoint, use $\inner{Z}{D_iX_i}=\inner{\sym(ZX_i^\top)}{D_i}$ for symmetric $D_i$. The tuple on the right sums to zero by centering, so it lies in the tangent space without an additional projection. Multiplying by $X_i$ and summing gives the Gram formula.

For one vector probe, $X_i=w_i\in\R^d$, the formula is
\[
 L_wL_w^*=\tfrac12\left[\left(\sum_i\norm{w_i}^2\right)I+\sum_iw_iw_i^\top\right].
\]
It is positive definite unless all centered values vanish. Thus one generic read has only $d$ observable directions, while the complete action tuple has $(n-1)s_d$ tangent coordinates. By using sufficient source-contrast probes, one recovers every $D_i$: for each $i<n$, a probe contrasting source $i$ with source $n$ and carrying the basis columns identifies $D_i-D_n$. These contrasts, together with $\sum_iD_i=0$, recover the tuple. Static pointwise observability can therefore be much smaller than the state needed for all probes.

\subsection{Individually attainable columns can be jointly incoherent}
Take two sources $V_1=I_2$, $V_2=-I_2$, and let $Y=(W_1-W_2)=2W_1-I_2$. Feasible $Y$ are exactly symmetric contractions. For $0<\gamma\le1$, consider the target
\[
 T=\gamma\begin{pmatrix}0&-1\\1&0\end{pmatrix}.
\]
Its first column is attained by the symmetric contraction $\begin{psmallmatrix}0&\gamma\\\gamma&0\end{psmallmatrix}$; its second by $\begin{psmallmatrix}0&-\gamma\\-\gamma&0\end{psmallmatrix}$. But no shared feasible operator attains both, since $T$ is antisymmetric. For any symmetric $Y$, Frobenius orthogonality gives
\[
 \norm{Y-T}_\F^2=\norm Y_\F^2+\norm T_\F^2\ge2\gamma^2,
\]
attained by $Y=0$, equivalently $W_1=W_2=I_2/2$. Solving each response separately does not construct a coherent reusable action. The example parallels the distinction in the main paper, but supplies no recurrent-memory lower bound by itself.

%% file: appendices/10_complete_attention.tex
\section{Complete Attention: normalization and reusable feature actions}\label{app:complete}
This appendix gives the operator construction underlying the observation map in Appendix~\ref{app:coherence}. Complete Attention represents a \emph{shared action} by positive matrices summing to the identity. Its observable coordinates, developed in Appendix~\ref{app:observable}, compress a specified collection of responses. Neither operation makes the resulting response coordinates sufficient for an additional, previously unspecified observation. We state the construction and its proofs independently of any architectural performance claim.

\subsection{The operator simplex and its entropy normalization}
For $n\ge2$, let $\mathbb H_d$ denote complex Hermitian matrices and set
\[
 \mathcal P_{n,d}=\{(W_i)_{i=1}^n:W_i\succeq0,\ \sum_iW_i=I_d\},
 \qquad U_i=I_d/n.
\]
Real symmetric matrices give the real version used for payloads below. The tuple acts by $v\mapsto\sum_iW_iv_i$. It preserves common values. Moreover $0\preceq W_i\preceq I$ implies $\sum_iW_i^2\preceq I$, so the block row $[W_1\ \cdots\ W_n]$ is a contraction:
\begin{equation}\label{eq:ca_contraction}
 \left\|\sum_iW_iv_i\right\|^2\le\sum_i\|v_i\|^2.
\end{equation}
Positive operator measures and constrained matrix-entropy optimization are established objects \citep{watrous2018,pavlov2024}; the statements below specify the normalization and coordinates used here.

\begin{theorem}[Complete Softmax]\label{thm:ca_normalizer}
For Hermitian scores $S_1,\ldots,S_n$ and $\tau>0$, there is a unique Hermitian $L=L_\tau(S)$ satisfying
\begin{equation}\label{eq:ca_normalization}
 W_i=\exp((S_i-L)/\tau),\qquad \sum_iW_i=I.
\end{equation}
The tuple $W=\operatorname{CSoftmax}_\tau(S)$ is the unique maximizer of
\[
 \Lambda_\tau(S)=\max_{W\in\mathcal P_{n,d}}
 \left\{\sum_i\operatorname{Tr}(W_iS_i)-\tau\sum_i\operatorname{Tr}(W_i\log W_i)\right\},
\]
where $0\log0=0$. It is strictly positive, depends analytically on the scores, and satisfies $\Lambda_\tau(S)=\operatorname{Tr}L$ and $\nabla_{S_i}\Lambda_\tau=W_i$.
\end{theorem}
\begin{proof}
The feasible set is compact. Matrix entropy is strictly convex, so the continuous objective has a unique maximizer. To check strict convexity and a useful stronger bound, diagonalize a positive $W\preceq I$. In that basis,
\[
 \langle H,D\log_W[H]\rangle
 =\sum_{j,k}\frac{\log\lambda_j-\log\lambda_k}{\lambda_j-\lambda_k}|H_{jk}|^2
 \ge\|H\|_F^2,
\]
with diagonal divided differences interpreted as $1/\lambda_j$. Integrating along segments proves strong convexity in the interior and, by continuity, strict convexity on the closure.

A maximizing tuple cannot have a zero eigenvalue. Replace it by $(1-t)W+tU$. For each positive eigenvalue its entropy changes by $O(t)$; each zero eigenvalue contributes $(t/n)\log(t/n)$. The score term changes by $O(t)$. Thus the negative-entropy gain is positive and of order $-t\log t$, contradicting optimality for small $t>0$.

At an interior maximizer, stationarity under $\sum_i\dot W_i=0$ gives
$S_i-\tau(\log W_i+I)=K$ for one Hermitian $K$. Put $L=K+\tau I$ to obtain \eqref{eq:ca_normalization}. Conversely, a tuple satisfying that equation is stationary and hence uniquely optimal. Its $L=S_i-\tau\log W_i$ is unique. Substitution and $\sum_iW_i=I$ give $\Lambda_\tau(S)=\operatorname{Tr}L$.

For analyticity, write $A_i=(S_i-L)/\tau$. The derivative with respect to $L$ of $\sum_i e^{A_i}-I$ is $-\tau^{-1}\mathcal H$, where
\[
 \mathcal H[H]=\sum_iD\exp_{A_i}[H],\qquad
 D\exp_A[H]=\int_0^1 e^{(1-t)A}He^{tA}\,dt.
\]
In an eigenbasis of $A$, every coefficient of this derivative is a strictly positive exponential divided difference. Hence $\mathcal H$ is positive definite and invertible on Hermitian matrices. The analytic implicit-function theorem gives local analyticity, and uniqueness joins the local solutions globally. Differentiating the optimized value gives the displayed gradient; the terms involving $\dot W$ cancel by stationarity and conservation.
\end{proof}

\subsection{Relative scores, inverse coordinates, and the feature algebra}
\begin{proposition}[Exact score and algebra identities]\label{prop:ca_algebra}
Complete Softmax is equivariant under unitary conjugation and direct sums, and is invariant under every common Hermitian shift $S_i\mapsto S_i+C$. It induces a bijection from score tuples modulo common shifts to the strictly positive operator simplex, with inverse $[S_i]=[\tau\log W_i]$. In particular,
\begin{equation}\label{eq:ca_logodds}
 \tau(\log W_i-\log W_j)=S_i-S_j,
 \qquad
 C^*(I,W_1,\ldots,W_n)=C^*(I,S_2-S_1,\ldots,S_n-S_1).
\end{equation}
\end{proposition}
\begin{proof}
Conjugation, direct sums, and simultaneous replacement of $(S,L)$ by $(S+C,L+C)$ preserve \eqref{eq:ca_normalization}; uniqueness proves the transformation laws. Conversely, $S_i=\tau\log W_i$ and $L=0$ normalize any interior tuple. Subtracting its logarithms proves the inverse and log-odds statements.

Let $\mathcal A=C^*(I,S_i-S_1:i\ge2)$ and first shift to $S_1=0$. Every unitary in the commutant $\mathcal A'$ fixes the scores, and therefore fixes every $W_i$ by uniqueness. The finite-dimensional bicommutant theorem gives $W_i\in\mathcal A$. Conversely \eqref{eq:ca_logodds} and finite-dimensional functional calculus put every relative score in the algebra generated by the $W_i$. This proves equality. Real scores give real effects by complex conjugation and uniqueness.
\end{proof}
The algebra identity concerns relative scores: noncommutativity caused only by a common shift has no effect. It does not assert that the integer-valued algebra type is stable under arbitrary perturbations.

A brief choice characterization explains the logarithms. Suppose relative evidence $O_d(A,B)$ on positive matrices obeys a cocycle, unitary/direct-sum naturality, scalar multiplicative composition, and scalar order. The cocycle writes it as $\theta_d(A)-\theta_d(B)$ with $\theta_d(A)=O_d(A,I)$. The scalar function $t\mapsto\theta_1(e^t)$ is additive and increasing, hence equals $ct$ for $c>0$. Direct sums and the spectral theorem then give
$O_d(A,B)=c(\log A-\log B)$.

Suppose further that all menu sizes share a natural pair map $O_d(W_i,W_j)=\Phi_d(S_i-S_j)$ and positive scalar score orientation. Applying the cocycle to the three scores $(A+B,B,0)$ gives additivity of $\Phi_d$ on arbitrary Hermitian matrices. To see that this is real linear without an unstated continuity assumption, write $f=\Phi_1$. Conjugate $\operatorname{diag}(x,y)$ by the two real rotations with first rows $(\sqrt t,\pm\sqrt{1-t})$. Their off-diagonal entries cancel on addition. Additivity, direct sums, and conjugation naturality give
\[
 f(tx+(1-t)y)=tf(x)+(1-t)f(y),\qquad 0\le t\le1.
\]
Here $f(2z)=2f(z)$ follows from additivity. Taking $y=0,x=1$ and extending by additivity yields $f(z)=\alpha z$ on all reals. Spectral decomposition gives $\Phi_d(A)=\alpha A$; orientation gives $\alpha>0$. Thus the rule is exactly \eqref{eq:ca_logodds} with $\tau=c/\alpha$, whose normalized solution is Theorem~\ref{thm:ca_normalizer}. These are explicit assumptions on a choice rule, not consequences of a generic attention architecture.

\subsection{Small-block implementation and a finite-score subspace router}
If $U\in\mathbb C^{d\times m}$ is an isometry and
$S_i=s_iI_d+UH_iU^*$, let $p_i=\operatorname{softmax}_i(s/\tau)$ and
$\widetilde W=\operatorname{CSoftmax}_\tau((s_iI_m+H_i)_i)$. Direct-sum naturality gives
\begin{align}\label{eq:ca_smallblock}
 W_i&=p_iI_d+U(\widetilde W_i-p_iI_m)U^*,\\
 \sum_iW_iv_i&=\sum_ip_iv_i+U\sum_i(\widetilde W_i-p_iI_m)U^*v_i.\nonumber
\end{align}
With the subspace and normalized blocks supplied, application uses $O(ndm+nm^2)$ arithmetic. Basis and score generation, dense $m\times m$ exponential evaluations, and normalization iterations are separate costs. The formula does not provide a free contextual frame.

For two sources the effects commute, since $W_2=I-W_1$. Equation~\eqref{eq:ca_logodds} gives
\[
 W_1=\operatorname{sigmoid}((S_1-S_2)/\tau),\qquad W_2=I-W_1,
\]
where the sigmoid is applied spectrally. Given a projector $P$, use score difference $\gamma(2P-I)$. Then
\[
 W_{1,\gamma}=\sigma(\gamma/\tau)P+\sigma(-\gamma/\tau)(I-P),
 \quad \|W_{1,\gamma}-P\|_\op=\sigma(-\gamma/\tau)\le e^{-\gamma/\tau}.
\]
The limiting binary action on $(x,-x)$ is $(2P-I)x$, with finite-score output error at most $2e^{-\gamma/\tau}\|x\|$. Exact boundary projection requires the limit, not finite logits. Although $C^*(I,P)$ has only one-dimensional irreducible blocks, specifying $P$ across contexts still costs information. Small local block size is not a small reusable memory code.

\subsection{Implicit differentiation}
For differentiable scores, differentiating \eqref{eq:ca_normalization} gives
\begin{equation}\label{eq:ca_implicit}
 \dot L=\mathcal H^{-1}\sum_iD\exp_{A_i}[\dot S_i],\qquad
 \dot W_i=\tau^{-1}D\exp_{A_i}[\dot S_i-\dot L].
\end{equation}
This follows by summing $\dot W_i$ and setting the sum to zero. Positive definiteness of $\mathcal H$ makes the solve well posed for finite scores. Its condition number need not be uniformly bounded near the simplex boundary. A numerical implementation must check the normalization residual rather than assuming that a fixed iteration count gives exact effects.

%% file: appendices/11_observable_coordinates.tex
\section{Observable Complete Attention coordinates and their exact limitation}\label{app:observable}
We now make precise the response compression used in the earlier operator construction. The result applies to the multi-probe map $L_X$ from Appendix~\ref{app:coherence}, including rank-deficient designs. It describes a \emph{fixed observation map}, not a state closed under arbitrary future probes.

\subsection{A complete coordinate system for the observable quotient}
Work over real symmetric matrices and let
\[
 \mathcal T=\{(D_i):D_i\in\Sym_d,\ \sum_iD_i=0\}.
\]
Let $L:\mathcal T\to\mathcal V$ be a surjective linear map onto its image, endowed with a Euclidean inner product; put $r_L=\dim\mathcal V$. Define
\[
 \mathcal C=L(\mathcal P_{n,d}-U),\qquad
 H_\tau(W)=\tau\sum_i\operatorname{Tr}\bigl(W_i\log(nW_i)\bigr),
\]
\[
 A_L(q)=\max_{W\in\mathcal P_{n,d}}
 \{\langle q,L(W-U)\rangle-H_\tau(W)\},\qquad q\in\mathcal V.
\]
The convention $0\log0=0$ is understood. The entropy satisfies $0\le H_\tau(W)\le\tau d\log n$: the lower bound follows by convexity and the uniform tuple, and the upper by $\lambda\log\lambda\le0$ for effect eigenvalues $\lambda\in[0,1]$.

\begin{theorem}[Complete observable coordinates]\label{thm:ca_coordinates}
Let $W(q)=\operatorname{CSoftmax}_\tau(L^*q)$ and $F_L(q)=L(W(q)-U)$. For $r_L>0$, the map
\[
 F_L:\mathcal V\longrightarrow\operatorname{relint}\mathcal C
\]
is an analytic diffeomorphism. It obeys
\begin{equation}\label{eq:ca_hessian}
 \nabla A_L=F_L,\qquad
 0\prec\nabla^2 A_L=L\mathcal H_{W,\mathcal T}^{-1}L^*
 \preceq\tau^{-1}LL^*,
\end{equation}
where $\mathcal H_{W,\mathcal T}$ is the Hessian of $H_\tau$ restricted to $\mathcal T$. Every boundary response is a limit of finite-query responses. No locally Lipschitz decoder on an open subset of $\R^m$, $m<r_L$, can cover this relative interior. When $r_L=0$, the only response is zero.
\end{theorem}
\begin{proof}
A neighborhood of $U$ in $U+\mathcal T$ is strictly positive, so its linear image contains a neighborhood of zero in $\mathcal V$. Thus $\mathcal C$ has nonempty interior relative to $\mathcal V$. The normalization theorem gives a unique positive optimizer, analytic in $q$. Differentiating its stationarity equation on $\mathcal T$ gives the Hessian identity. Entropy strong convexity gives $\mathcal H_{W,\mathcal T}\succeq\tau I$, and surjectivity of $L$ makes $L^*$ injective. This proves both inequalities in \eqref{eq:ca_hessian}. Integrating the positive Hessian along any line proves strict monotonicity and hence injectivity of $F_L$.

For surjectivity, take $z\in\operatorname{relint}\mathcal C$. If $z\ne0$, then $(1+\epsilon)z\in\mathcal C$ for some $\epsilon>0$. A feasible lift of this extrapolated point, mixed with $U$ in proportions $1/(1+\epsilon)$ and $\epsilon/(1+\epsilon)$, is a strictly positive lift of $z$. For $z=0$, use $U$. On the compact fiber producing $z$, entropy has a unique minimizer $W^*(z)$. It is positive: mixing any singular candidate toward the positive lift gives an entropy decrease of order $t\log t$ on zero eigenspaces, while all positive eigenvalues contribute $O(t)$. More explicitly, if the positive lift has smallest eigenvalue at least $c>0$, the eigenvalues emerging from zero lie between $ct$ and $C t$ by eigenvalue bounds; their $t\log t$ contribution dominates all $O(t)$ terms.

The equality-constraint multiplier theorem on this positive fiber gives $\nabla_{\mathcal T}H_\tau(W^*)=L^*q$ for some $q\in\mathcal V$: the derivative annihilates $\ker L$, whose orthogonal complement is $\operatorname{im}L^*$. The normalization theorem identifies $W^*$ with $W(q)$, proving surjectivity. The inverse-function theorem and \eqref{eq:ca_hessian} give an analytic inverse. Convex mixing of any boundary point with zero produces interior points converging to it.

Finally, a locally Lipschitz map from an open subset of $\R^m$ is covered by countably many Lipschitz restrictions to compact cubes. When $m<r_L$, their images have zero $r_L$-dimensional volume, by subdividing into cubes of side $\delta$ and letting $\delta\downarrow0$. Their union cannot contain a nonempty open subset of $\mathcal V$. This is a regular-coordinate lower bound, separate from the continuous encoder lower bound of Lemma~\ref{lem:dimension}.
\end{proof}

For $L=L_X$, use $\mathcal V=\operatorname{im}L_X\subset\R^{d\times b}$ and Frobenius inner products. Centering the source matrices gives
\[
 (L_X^*Q)_i=\operatorname{sym}(QX_i^\top),\qquad
 \operatorname{rank}(L_X^*Q)_i\le2b.
\]
The coordinate count is $\operatorname{rank}L_X$, not automatically $db$. Rank deficiency removes only presently unobservable directions. For one vector probe $w=(w_i)$ with $\sum_iw_i=0$ and $s_w=\sum_i\|w_i\|^2>0$, Appendix~\ref{app:coherence} gives
\[
 K_w=L_wL_w^*=\tfrac12\left(s_wI+\sum_iw_iw_i^\top\right),
 \qquad (s_w/2)I\preceq K_w\preceq s_wI.
\]
Thus one $d$-vector query, with rank-two scores $S_i=\operatorname{sym}(qw_i^\top)$, covers every interior response on those particular values.

\subsection{Observed-output learning and regularized coordinates}
Let $\Omega_L(z)=\min\{H_\tau(W):L(W-U)=z\}$ on $\mathcal C$ and $+\infty$ elsewhere. Compactness gives attainment; strict convexity gives a unique lift. The function is convex and lower semicontinuous, and $A_L=\Omega_L^*$. For an observed feasible response $z$, the loss
\[
 \ell(q;z)=A_L(q)-\langle q,z\rangle
\]
has gradient $F_L(q)-z$ and needs no labels for the individual operators. This is a specialization of Fenchel--Young learning \citep{blondel2020}, not a new convex-duality principle.

\begin{proposition}[Exact calibration and a regularized inverse]\label{prop:ca_calibration}
For $z\in\mathcal C$ and any finite $q$,
\begin{align}\label{eq:ca_calibration}
 \Delta_L(q;z)&:=A_L(q)-\langle q,z\rangle+\Omega_L(z)
 =\tau\sum_i D_q(W_i^*(z)\|W_i(q)),\\
 \Delta_L(q;z)&\ge\frac{\tau}{2\|L\|^2}\|F_L(q)-z\|^2\quad(L\ne0),\nonumber
\end{align}
where $D_q(A\|B)=\operatorname{Tr}(A(\log A-\log B)-A+B)$. For any $z_0\in\mathcal V$ and $\epsilon>0$, the unique minimizer
\[
 q_\epsilon=\argmin_q\{A_L(q)-\langle q,z_0\rangle+\epsilon\|q\|^2/2\}
\]
satisfies
\begin{equation}\label{eq:ca_prox}
 F_L(q_\epsilon)=\argmin_{z\in\mathcal C}
 \{\Omega_L(z)+\|z-z_0\|^2/(2\epsilon)\},\quad
 \|F_L(q_\epsilon)-\operatorname{Proj}_{\mathcal C}z_0\|^2
 \le2\epsilon\tau d\log n.
\end{equation}
\end{proposition}
\begin{proof}
Subtract the entropy value at its optimizer and use stationarity on $\mathcal T$. The resulting entropy Bregman divergence is the sum of $\tau D_q$; the identity-sum constraint cancels its linear trace terms. Entropy strong convexity bounds it below by $\tau\|W^*(z)-W(q)\|_F^2/2$. Applying $L$ gives \eqref{eq:ca_calibration}.

The regularized objective is $\epsilon$-strongly convex and coercive, because $A_L$ has an affine lower bound. Stationarity gives $z_0=F_L(q_\epsilon)+\epsilon q_\epsilon$. Fenchel conjugacy says $q_\epsilon\in\partial\Omega_L(F_L(q_\epsilon))$, which is exactly the proximal optimality condition in \eqref{eq:ca_prox}. If $p=\operatorname{Proj}_{\mathcal C}z_0$, compare the proximal objective with its value at $p$. Since $0\le\Omega_L\le\tau d\log n$ and the projection inequality gives
$\|z-z_0\|^2\ge\|p-z_0\|^2+\|z-p\|^2$ for $z\in\mathcal C$, the bound follows. Strong convexity and the analytic implicit-function theorem also make $q_\epsilon$ analytic in $z_0$ for fixed $L$.
\end{proof}
Convexity is in the query, or in an affine query parameterization. It does not imply global convergence for a nonlinear shared query network. More importantly, learning the complete set of present responses does not identify $\ker L$. The continuation criterion in Appendix~\ref{app:coherence} still requires those directions whenever an additional observation $M$ has $\ker L\not\subseteq\ker M$.

\subsection{The canonical action is supported by the observed geometry}
For a single probe, set $\mathcal K_z=\operatorname{span}\{w_1,\ldots,w_n,z\}$ and let $P_z$ be its orthogonal projector. Then the entropy-minimizing lift obeys
\begin{equation}\label{eq:ca_intrinsic}
 W_i^*(z)=P_zW_i^*(z)P_z+(I-P_z)/n,
 \qquad\dim\mathcal K_z\le\min(d,n).
\end{equation}
Indeed, reflection across $\mathcal K_z$ fixes the sources and response. It maps the feasible fiber to itself and preserves entropy; uniqueness forces the lift to be block diagonal. The complementary blocks are unconstrained except for their identity sum and therefore minimize entropy at $I/n$. The source span has dimension at most $n-1$, proving the dimension bound. Consequently the entropy oscillation used in \eqref{eq:ca_prox} improves to $\tau\min(d,n)\log n$ in this single-probe setting.

For a finite query and $z=F_w(q)$, one also has
$\operatorname{span}\{w_i,z\}=\operatorname{span}\{w_i,q\}$. Reflection across the latter span fixes the scores, so it fixes $z$ and gives one inclusion. Reflection across the former span fixes $w$ and $z$; orthogonal covariance gives $F_w(Jq)=JF_w(q)=z$. Injectivity of $F_w$ forces $Jq=q$, giving the other inclusion. Thus an active basis can be computed from the current values and query. This is not identification of a history-dependent teacher action outside that span.

%% file: appendices/12_response_geometry.tex
\section{Exact response geometry and an executable binary construction}\label{app:geometry}
This section records the response geometry without assuming that individual effects are observed. It supplies constructive versions of the quotient in Appendix~\ref{app:observable}; it does not change the continuation-state counts.

\subsection{An attained support formula for any number of sources}
For $v_1,\ldots,v_n\in\R^d$, let
\[
 \mathcal B(v)=\left\{\sum_iW_iv_i:W\in\mathcal P_{n,d},\ W_i\in\Sym_d\right\}.
\]
For $p$ in the probability simplex, write
$m_p=\sum_ip_iv_i$ and $\Sigma_p=\sum_ip_i(v_i-m_p)(v_i-m_p)^\top$.
\begin{theorem}[Transverse-covariance support]\label{thm:ca_support}
For $q\ne0$, set $e=q/\|q\|$ and $P_e^\perp=I-ee^\top$. Then
\begin{equation}\label{eq:ca_support}
 h_{\mathcal B(v)}(q)=\max_{p\in\Delta_{n-1}}
 \left\{q^\top m_p+\|q\|\operatorname{Tr}
 \sqrt{P_e^\perp\Sigma_pP_e^\perp}\right\}.
\end{equation}
Every maximum has a feasible attaining effect tuple, including when some $p_i=0$. In particular,
\begin{equation}\label{eq:ca_binaryball}
 \mathcal B(v_1,v_2)=\frac{v_1+v_2}{2}
 +\frac{\|v_1-v_2\|}{2}\,\mathbb B_2^d.
\end{equation}
\end{theorem}
\begin{proof}
Use coordinates $\R e\oplus e^\perp$ and write each effect as
$W_i=\begin{psmallmatrix}p_i&z_i^\top\\z_i&B_i\end{psmallmatrix}$.
Positivity and conservation require $p_i\ge0$, $\sum_ip_i=1$, $\sum_i z_i=0$, $\sum_iB_i=I_{e^\perp}$, and $B_i\succeq z_iz_i^\top/p_i$ for $p_i>0$. If $p_i=0$, positivity forces $z_i=0$.
For each positive $p_i$, put $Z_i=z_i/\sqrt{p_i}$ and put $Z_i=0$ otherwise. The matrix $Z$ of these columns then satisfies
\[
 ZZ^\top\preceq I_{e^\perp},\qquad Z\sqrt p=0.
\]
The support objective divided by $\|q\|$ is
$e^\top m_p+\sum_i z_i^\top P_e^\perp v_i$.
Centering uses $\sum_i z_i=0$ to write its second term as $\operatorname{Tr}(Z^\top D_p)$, where column $i$ of $D_p$ is
$\sqrt{p_i}P_e^\perp(v_i-m_p)$. Since $D_p\sqrt p=0$, a polar maximizer
$Z=(D_pD_p^\top)^{\dagger/2}D_p$ has the required kernel, is a contraction, and attains the nuclear norm $\|D_p\|_*$.
This is exactly $\operatorname{Tr}\sqrt{P_e^\perp\Sigma_pP_e^\perp}$.

To attain it by effects, take $z_i=\sqrt{p_i}Z_i$ and, for $p_i>0$, set
\[
 B_i=Z_iZ_i^\top+p_i(I-ZZ^\top).
\]
For $p_i=0$, take $W_i=0$; the corresponding column of $D_p$, and hence of $Z$, is zero. The Schur complements are positive and the blocks sum to identity. This gives feasible effects attaining the fixed-$p$ value. The objective in \eqref{eq:ca_support} is continuous on the compact simplex, so it has a maximizer.

For two sources, put $m=(v_1+v_2)/2$ and $a=(v_1-v_2)/2$. Feasible outputs are exactly $m+Aa$ for symmetric contractions $A=2W_1-I$. Therefore their support above $q^\top m$ is
\[
 \max_{-I\preceq A\preceq I}\operatorname{Tr}\bigl(A\operatorname{sym}(aq^\top)\bigr)
 =\|\operatorname{sym}(aq^\top)\|_*
 =\|a\|\|q\|.
\]
The possibly nonzero eigenvalues of the symmetric rank-two matrix are $(a^\top q\pm\|a\|\|q\|)/2$, proving the last equality. Compact convex sets are determined by their support functions, which yields \eqref{eq:ca_binaryball}. Zero contrast gives a singleton, as required.
\end{proof}

The body is not generally obtained by convexifying its two-source bodies. For three equilateral unit vectors
\[
 v_1=(1,0),\quad v_2=(-1/2,\sqrt3/2),\quad v_3=(-1/2,-\sqrt3/2),
\]
the effects
\[
 W_1=\begin{pmatrix}11/12&0\\0&0\end{pmatrix},\qquad
 W_{2,3}=\begin{pmatrix}1/24&\pm\sqrt3/12\\\pm\sqrt3/12&1/2\end{pmatrix}
\]
are positive, sum to identity, and yield $(9/8,0)$. The dual matrix $D=\operatorname{diag}(1,1/8)$ satisfies
$D\succeq\operatorname{sym}(e_1v_i^\top)$ for all $i$: each difference has nonnegative diagonal entries and determinant zero. Thus
$\sum_i\operatorname{Tr}(W_i\operatorname{sym}(e_1v_i^\top))\le\operatorname{Tr}D=9/8$, proving optimality. The largest pairwise-ball support is instead $(1+2\sqrt3)/4<9/8$.

\subsection{Matrix-free binary Complete Attention}
For two values, define $m=(v_1+v_2)/2$, $a=(v_1-v_2)/2$, $R=\|a\|$, and a scalar prior log-odds $\ell$. Use score difference
\[
 S_1-S_2=\tau\ell I+qa^\top+aq^\top,
 \qquad G=\operatorname{sigmoid}((S_1-S_2)/\tau).
\]
The output $Gv_1+(I-G)v_2$ admits a closed formula. For $t=\|q\|>0$, put
\[
 z_\pm=\ell/2+(q^\top a\pm Rt)/(2\tau),\qquad T_\pm=\tanh z_\pm.
\]
Then
\begin{equation}\label{eq:ca_binaryformula}
 y=m+\tfrac12(T_++T_-)a+\frac{R}{2}(T_+-T_-)\frac q t.
\end{equation}
For $t=0$ its continuous limit is $m+\tanh(\ell/2)a$, and for $R=0$ it is $m$. To prove the formula, note that $y=m+\tanh(\ell I/2+(qa^\top+aq^\top)/(2\tau))a$. The rank-two part has eigenvalues $(q^\top a\pm Rt)/(2\tau)$. Spectral interpolation on their invariant plane gives \eqref{eq:ca_binaryformula}, with the continuous limit covering degeneracy. Evaluation takes $O(d)$ arithmetic after the query is supplied. The zero-query derivative is
\[
 D_qy\big|_{q=0}=\frac{\operatorname{sech}^2(\ell/2)}{2\tau}(R^2I+aa^\top),
\]
which follows by expanding the scalar hyperbolic tangents. A divided-difference implementation avoids division by a small $t$.

An existing scalar attention row can use this binary construction on two summaries. Partition its keys into nonempty groups and define
\[
 \pi_g=\sum_{i\in G_g}p_i,\qquad z_g=\pi_g^{-1}\sum_{i\in G_g}p_iv_i.
\]
Apply the binary formula to $(z_1,z_2)$ with $\ell=\log(\pi_1/\pi_2)$. At $q=0$ it returns exactly the old scalar attention output. At other finite queries the induced source effects are $(p_i/\pi_1)G$ or $(p_i/\pi_2)(I-G)$, hence positive and identity-summing. They all commute. This is a restricted two-summary construction, not the entire multi-source body in Theorem~\ref{thm:ca_support}; the equilateral example proves strict loss of responses even after arbitrary convex combination of pairwise choices.

\subsection{Completeness does not privilege one response parameterization}
The same binary ball has the analytic coordinates
\[
 y=m+R\,z/\sqrt{1+\|z\|^2},\qquad z\in\R^d.
\]
They use the same dimension and $O(d)$ arithmetic without entropy normalization. For an explicit positive lift, let $u=a/R$, $t=(y-m)/R$, $\alpha=u^\top t$, and $b=t-\alpha u$. When $\|t\|<1$, set
\[
 A=\alpha uu^\top+ub^\top+bu^\top-\frac{\alpha}{1-\alpha^2}bb^\top.
\]
Then $Au=t$. The Schur complements of the leading scalar blocks of $I\pm A$ are both $I_{u^\perp}-bb^\top/(1-\alpha^2)\succ0$, so $-I\prec A\prec I$. The effects $(I\pm A)/2$ produce the desired output. This establishes a same-range alternative and prevents response completeness from being interpreted as unique architectural necessity. Neither parameterization preserves an unknown operator on additional payloads without the extra evidence required by Appendix~\ref{app:coherence}.

%% file: appendices/13_reusable_execution.tex
\section{Reusable scalar execution, residuals, and contextual coordinates}\label{app:execution}
The following results retain a complementary part of the routing theory: once a shared action is specified, what does its execution require? This is not the memory required to identify that action from a stream. Tensor rank is used in its ordinary bilinear-complexity sense \citep{koiran2020}.

\subsection{Exact scalar cost above a prescribed baseline}
Let $Y_1,\ldots,Y_n\in\R^{d\times d}$ with $\sum_iY_i=I$ and choose $\beta_i>0$, $\sum_i\beta_i=1$. Write
$b_\beta=\sum_i\beta_iv_i$ and $C_i=Y_i-\beta_iI$. The real CP rank of the slice tensor $C$ is the smallest $r$ such that
\[
 C_i=\sum_{h=1}^r a_{ih}u_hv_h^\top.
\]
Consider a centered projected-scalar execution
\begin{equation}\label{eq:ca_scalarcompiler}
 b_\beta+\sum_{h=1}^M B_h\sum_i p_{ih}(v_i-b_\beta),
 \qquad\operatorname{rank}B_h\le1,\quad p_{ih}>0,\quad\sum_ip_{ih}=1.
\end{equation}
Factors may depend on a fixed context, but not on the fresh payloads in this formula.
\begin{proposition}[Residual value cost]\label{prop:ca_cp}
The exact minimum number of independent value channels in \eqref{eq:ca_scalarcompiler} is $\operatorname{rank}_{\mathrm{CP},\R}C$. All scalar probabilities can be implemented by finite logits.
\end{proposition}
\begin{proof}
Expansion gives $C_i=\sum_h(p_{ih}-\beta_i)B_h$, so its CP rank is at most $M$. Conversely, start from a minimal CP decomposition. Apply $I-\beta\mathbf1^\top$ to each source factor $a_h$. Since $\sum_iC_i=0$, the tensor is unchanged and the factors now have zero source sum. Delete any zero terms. For each remaining term choose
$t_h>\max(0,\max_i(-a_{ih}/\beta_i))$ and set
\[
 p_{ih}=\beta_i+a_{ih}/t_h,\qquad B_h=t_hu_hv_h^\top.
\]
These probabilities are strictly positive, have source sum one, and reconstruct the same term. Scalar logits $\log p_{ih}$ realize them. This proves equality, including the zero-rank case.
\end{proof}
A physical head with several value channels shares one probability vector across them; it is more restricted than independent channels. The proposition counts value channels, not head count, score-generation work, or history storage. The scaling $t_h$ can also increase cancellation. Finite exact width is not a uniform conditioning claim.

\subsection{A sharp finite-probability compiler for positive actions}
Positivity permits a stronger universal bound than a generic signed matrix-unit expansion. Put $s_d=d(d+1)/2$. For $d\ge2$, use the fixed rank-one basis of $\Sym_d$
\[
 b_j=e_j,\qquad b_{jk}=e_j+e_k\ (j<k),\qquad B_h=b_hb_h^\top.
\]
Its coordinates are
\[
 c_{jk}(A)=A_{jk},\qquad c_j(A)=A_{jj}-\sum_{k\ne j}A_{jk},
 \qquad A=\sum_hc_h(A)B_h.
\]
These coefficients can have both signs; the construction aligns their signs across sources rather than assuming positive dictionary atoms.

\begin{proposition}[Universal positive scalar compilation]\label{prop:ca_sharpcompiler}
Every finite real tuple $Y_i\succ0$, $\sum_iY_i=I_d$, admits
\[
 Y_i=\sum_{h=1}^{s_d}p_{ih}A_h,
 \qquad \rank A_h=1,\quad p_{ih}>0,\quad\sum_ip_{ih}=1.
\]
Thus $s_d$ independent scalar value channels suffice, with finite logits. If $\operatorname{span}\{Y_i\}=\Sym_d$, no smaller number of channels suffices, even with signed scalar coefficients. For a compact continuous family of interior tuples, one fixed rank-one dictionary suffices throughout, with continuous positive probabilities.
\end{proposition}
\begin{proof}
For $d\ge3$ choose $u=(1,\ldots,1)^\top$; for $d=2$ choose $u=(1,2)^\top$. Every $\eta_h=c_h(uu^\top)$ is nonzero. For $d\ge3$, its diagonal coordinates are $2-d$ and its off-diagonal coordinates are one. For $d=2$, the coordinates are $(-1,2,2)$.

Set $H=I+tuu^\top$. For each source and basis coordinate,
\[
 c_h(HY_iH)=t^2(u^\top Y_i u)\eta_h
 +t\,c_h(uu^\top Y_i+Y_i uu^\top)+c_h(Y_i).
\]
The leading coefficient is nonzero and has sign $\eta_h$, because $Y_i\succ0$. One sufficiently large finite $t$ makes every polynomial have that sign. Define
\[
 \lambda_{ih}=c_h(HY_iH),\quad
 \alpha_h=\sum_i\lambda_{ih}=c_h(H^2),\quad
 p_{ih}=\lambda_{ih}/\alpha_h,
 \qquad v_h=H^{-1}b_h.
\]
All $p_{ih}$ are positive and sum to one across sources. Congruencing the coordinate expansion by $H^{-1}$ gives
\[
 Y_i=\sum_h\lambda_{ih}v_hv_h^\top
 =\sum_hp_{ih}A_h,\qquad A_h=\alpha_hv_hv_h^\top.
\]
These are rank-one projected value maps and sum to identity. Logits $\log p_{ih}$ are finite. If the $Y_i$ span $\Sym_d$, any dictionary realizing them must have at least its dimension $s_d$; no sign or normalization assumption is needed for this lower bound. Dimension one is the ordinary scalar case.

On a compact continuous interior family, $u^\top Y_i u$ has a uniform positive lower bound, while all lower-order polynomial coefficients are bounded. Therefore one finite $t$ works everywhere. The resulting $A_h$ depend only on that $t$ and the fixed basis, whereas $p_{ih}$ depend continuously on the effects. This proves the final assertion.
\end{proof}
The atoms $A_h$ may be negative or nonsymmetric after an additional common similarity transformation. They are projected scalar value maps, not positive effects. The theorem is an existence and width result: its sign-aligning congruence can amplify cancellation, and producing the probabilities still requires the target action. It neither gives a dimension-independent condition number nor identifies an unknown action from its current output. It also rules out an unlimited-width separation from projected scalar arithmetic.

\subsection{Pointwise factors need not be continuously selectable}
A complementary issue arises when one controller must produce factors continuously across contexts. Write
$\mathcal P_d=\{uu^\top:u\in S^{d-1}\}$ and consider $T(P,x)=Px$. At each fixed $P$ this action has rank one. Nevertheless:
\begin{proposition}[Continuous contextual bottleneck]\label{prop:ca_contextual}
For $d\ge2$, continuous encoders $E:\mathcal P_d\times\mathbb B_2^d\to\R^m$ and arbitrary decoders retaining $P$ satisfy
\[
 \inf_{E,D}\sup_{P,\|x\|\le1}\|D(P,E(P,x))-Px\|
 =\begin{cases}1,&m<d,\\0,&m\ge d.\end{cases}
\]
The same equality holds when decoders are required to be continuous.
\end{proposition}
\begin{proof}
For $m<d$, the map $g(u)=E(uu^\top,u)-E(uu^\top,-u)$ is continuous and odd on $S^{d-1}$. Borsuk--Ulam gives $g(u)=0$ for some $u$ \citep{matousek2003}. The common context and code have one decoded value $z$, whereas the two targets are $u$ and $-u$. Since $2\le\|z-u\|+\|z+u\|$, one error is at least one. The zero decoder attains this error. For $m=d$, encode $Px$ and decode it directly; larger codes can be padded.
\end{proof}
Here context is retained without charge, while \emph{all payload-dependent information} must pass through $E$. A full raw-payload residual changes that model. Likewise, Proposition~\ref{prop:ca_cp} allows a separate factor choice at each context; it does not guarantee a globally continuous factor generator. Thus present response dimension, continuous contextual value width, and recurrent evidence dimension answer distinct questions even when all three involve matrices of low pointwise rank.

%% file: appendices/14_information_fusion.tex
\section{Information-weighted fusion: a direct bridge to memory}\label{app:fusion}
The operator construction has an exact connection to accumulation of statistical evidence, rather than only to a new response parameterization. The following calculations use the classical linear unbiased fusion model and distinguish its normalized action from the information needed to revise it.

\subsection{Optimal fusion and its unnormalized state}
Suppose $v_i=x+\xi_i$ with centered, mutually uncorrelated errors and covariances $C_i\succ0$. Put $P_i=C_i^{-1}$, $\Pi=\sum_iP_i$, and
\[
 b=\sum_iP_iv_i,\qquad\widehat x=\Pi^{-1}b,\qquad W_i=\Pi^{-1}P_i.
\]
The gains sum to identity. For any other fixed conserving gains $W_i+D_i$, $\sum_iD_i=0$, covariance expansion gives
\begin{equation}\label{eq:ca_blue}
 \operatorname{Cov}\!\left(\sum_i(W_i+D_i)v_i-x\right)
 =\Pi^{-1}+\sum_iD_iC_iD_i^\top.
\end{equation}
Indeed $W_iC_i=\Pi^{-1}$, so both cross terms vanish on summation. This proves optimality in the positive-semidefinite order, without a Gaussian assumption. The statement is about fixed linear unbiased estimators, not all predictors.

In the total-precision metric, the normalized effects are
\[
 E_i=\Pi^{-1/2}P_i\Pi^{-1/2}\succ0,\qquad\sum_iE_i=I,
 \qquad W_i=\Pi^{-1/2}E_i\Pi^{1/2}.
\]
They have finite Complete Softmax scores $S_i=\tau\log E_i$. This is a change of metric: the Euclidean gains $W_i$ need not be symmetric when precisions do not commute. An additional independent sensor with precision $A\succ0$ and observation $y$ updates the unnormalized state by
\begin{equation}\label{eq:ca_fusionupdate}
 (\Pi,b)\longmapsto(\Pi+A,b+Ay),\qquad
 \widehat x^+=(\Pi+A)^{-1}(\Pi\widehat x+Ay).
\end{equation}
The present estimate alone does not specify this update.

\begin{proposition}[A fixed complete fusion action can hide $s_d$ evidence coordinates]\label{prop:ca_fusionhidden}
Fix positive scalar weights $\beta_i$ summing to one. For arbitrary $\Pi\succ0$ take $P_i=\beta_i\Pi$ and zero observed vectors. Every member of this family has the same full fusion action $W_i=\beta_iI$, the same normalized effects $E_i=\beta_iI$, and the same current estimate zero. Any continuous summary that supports every next observation $y$ from one sensor of fixed precision $A=I$ needs at least $s_d=d(d+1)/2$ coordinates. This count is attained on the family by storing $\Pi$.
\end{proposition}
\begin{proof}
The equal-action assertions follow by substitution. After the added sensor, the response is $(\Pi+I)^{-1}y$. Equality for all $y$ identifies the invertible matrix $(\Pi+I)^{-1}$ and hence $\Pi$. The positive-definite cone is an open $s_d$-dimensional family, so Lemma~\ref{lem:dimension} proves necessity. Equation~\eqref{eq:ca_fusionupdate} with old $b=0$ proves attainment.
\end{proof}
This is the fusion analogue of the ridge zero-label family. It shows why even retaining the \emph{entire normalized operator action}, rather than only one output vector, can fail to retain the evidence needed for a new observation. It does not claim that sensor precision inputs are the same observation model as ridge keys.

\subsection{Feature algebras and an exact three-source comparison}
A different resource is the feature algebra permitted in a \emph{known} shared action. For centered complex payloads with $\E[v_iv_j^*]=\delta_{ij}I$, the squared prediction risk of $W$ relative to target effects $Y$ is
\[
 \mathcal R_Y(W)=\E\left\|\sum_i(W_i-Y_i)v_i\right\|^2
 =\sum_i\|W_i-Y_i\|_F^2.
\]
For a unital $*$-subalgebra $\mathcal A$, its trace-preserving conditional expectation is an orthogonal projection in the Frobenius inner product and preserves positivity and identity. Therefore it preserves feasibility, and
\begin{equation}\label{eq:ca_algebrarisk}
 \min_{W_i\in\mathcal A,\,W\in\mathcal P_{n,d}}\mathcal R_Y(W)
 =\sum_i\|Y_i-\mathbb E_{\mathcal A}Y_i\|_F^2,
 \qquad W_i^{\rm opt}=\mathbb E_{\mathcal A}Y_i.
\end{equation}
This follows by orthogonal Pythagoras. In finite dimensions the conditional expectation can be obtained by averaging conjugations by unitaries in the commutant; this also verifies positivity and trace preservation.

An exact example distinguishes operator routing, a single common basis, and convex ensembles of common bases. Put
\[
 Z=\begin{pmatrix}1&0\\0&-1\end{pmatrix},\quad
 X=\begin{pmatrix}0&1\\1&0\end{pmatrix},\quad
 H_1=Z,\quad H_{2,3}=(-Z\pm\sqrt3X)/2,
\]
\[
 Y_i(a)=(I+aH_i)/3,
\]
where $0<a<1$. The effects are strictly positive, sum to identity, and have finite scores $\tau\operatorname{arctanh}(a)H_i$. Their exponentials are proportional to $I+aH_i$ because $H_i^2=I$ and $\sum_iH_i=0$.

Let $\mathcal C$ denote the commuting tuples in $\mathcal P_{3,2}$. Then
\begin{align}\label{eq:ca_trine}
 \min_{W\in\mathcal P_{3,2}}\mathcal R_{Y(a)}(W)&=0,\qquad
 \min_{W\in\mathcal C}\mathcal R_{Y(a)}(W)=a^2/3,\\
 \min_{p\in\Delta_2}\mathcal R_{Y(a)}((p_iI)_i)&=2a^2/3,\qquad
 \min_{W\in\operatorname{conv}\mathcal C}\mathcal R_{Y(a)}(W)
 =\tfrac23(a-\sqrt3/2)_+^2.\nonumber
\end{align}
To verify the common-basis value, write $H_i=n_i\cdot\sigma$ in the Pauli representation. The three $n_i$ lie in a plane and satisfy $\sum_i n_in_i^\top=(3/2)P_{\rm plane}$. Pinching into the basis with axis $u$ projects $n_i$ onto $u$. Equation~\eqref{eq:ca_algebrarisk} gives
\[
 \frac{2a^2}{9}\sum_i(1-(n_i\cdot u)^2)
 =\frac{2a^2}{3}-\frac{a^2}{3}\|P_{\rm plane}u\|^2,
\]
minimized at $a^2/3$. Restricting to scalars removes the full traceless part and gives $2a^2/3$.

For the convex-ensemble result, any commuting tuple is
$W_i=p_iP_u+q_i(I-P_u)$ for probability vectors $p,q$. Therefore
\[
 \sum_i\operatorname{Tr}(H_iW_i)
 =\sum_i(p_i-q_i)n_i\cdot u
 \le\max_{i,j}\|n_i-n_j\|=\sqrt3.
\]
The inequality persists under convex averaging. The target has witness value $2a$ and $\sum_i\|H_i\|_F^2=6$, so Cauchy--Schwarz gives the lower bound in \eqref{eq:ca_trine}. To attain it, for each unordered pair $i,j$ define a commuting tuple $T^{ij}$ by
\[
 T_i^{ij}=\tfrac12\left(I+\frac{H_i-H_j}{\sqrt3}\right),\qquad
 T_j^{ij}=I-T_i^{ij},\qquad T_k^{ij}=0\quad(k\notin\{i,j\}).
\]
Since $(H_i-H_j)^2=3I$, these are complementary projectors. Averaging the three tuples gives exactly $Y(\sqrt3/2)$. Mixing that average with $U$ attains every $Y(a)$ below the threshold. Above the threshold, $Y(\sqrt3/2)$ attains the witness lower bound. This proves all equalities.

The projective-simulability threshold is established in the measurement-simulation literature \citep{oszmaniec2017}; it is not claimed as a new threshold here. The exact squared-risk calculation specifies the present routing comparison. Arbitrary signed projected scalar heads are not commuting positive tuples and are not covered by this comparison. The finite scalar compiler in Appendix~\ref{app:execution} explicitly shows their escape. Nor does a risk gap for a known action establish an additional memory lower bound; Proposition~\ref{prop:ca_fusionhidden} is the separate continuation statement.

%% file: appendices/09_scope.tex
\section{Assumptions and reproducible verification}\label{app:scope}
\subsection{Observation domains and counted resources}
The residual-loss classification fixes the regularizer, uses unrestricted real features and labels, and allows arbitrary history length. Every sample in the main residual and logistic theorems has unit weight. Strong convexity supplies existence, uniqueness, and smooth dependence; global smoothness supplies the tangent and differentiation arguments. The bounded logistic construction separately restricts features and labels and proves its own native identification. The generic kernel theorem uses positive weights; the exponential-kernel theorem removes them by a separate key-identification argument.

Exact continuous-state bounds count real coordinates, not bits. They exclude discontinuous infinite-precision encodings but do not restrict decoder complexity. Each lower-bound family occurs at a fixed length, so a known clock does not encode its distinguishing parameters. Bounds on fixed-predictor subfamilies concern prefix summaries there, not necessarily globally valid recurrences of the smaller size. The finite-bit theorem instead allows arbitrary measurable randomized encoders and decoders; its code is formed before independent future probes arrive. Its upper bound is an offline quantizer, not an efficient streaming scheme.

The operator appendices concern a specified positive, identity-summing action. Complete Softmax, observable response dimension, scalar value-channel width, contextual continuity, and recurrent memory have separate definitions. The scalar compiler permits signed value/output maps even though its source probabilities are positive. The trine comparison instead restricts the effective action to commuting positive tuples or their convex hull. The fusion calculation concerns fixed linear unbiased gains and explicitly states its noise moments. None of these results is assigned to an unrestricted Transformer without the required information-path reduction.

\subsection{Symbolic and numerical checks}
The supplementary archive contains two deterministic scripts. \texttt{checks/verify.py} tests the exact ridge collision, the update and gain identities, polynomial tangent and balancing ranks on finite-field witnesses, bounded logistic equal-jet constructions, the random-probe metric, the constants in the conditioning-volume proof, attention evidence recovery, and the shared-operator coherence example. Exact calculations use rational or finite-field arithmetic where indicated; numerical calculations use fixed seeds.

\texttt{checks/verify\_operators.py} checks the Complete Softmax normalization and relative-score identities, common-shift and orthogonal covariance, implicit differentiation, the observable Gram formula, the binary closed form and radial lift, the equilateral support witness, positive scalar compilation, the trine ensemble construction, and the hidden-precision fusion family. The solver raises an error if its declared normalization tolerance is not reached. Its small-matrix tests establish the listed identities in those instances, not convergence or stability at every possible score scale.

Both scripts write machine-readable records and fail on a violated assertion. Dependency versions and commands are supplied in the supplementary README. They generate algebraic test instances, not a benchmark dataset; no neural network is trained. The mathematical arguments in the preceding appendices, rather than a count of passing tests, establish the claims.